\documentclass[conference,a4paper]{IEEEtran}

\usepackage[utf8]{inputenc} 
\usepackage[T1]{fontenc}
\usepackage{url}
\usepackage{ifthen}
\usepackage{cite}
\usepackage{amsmath, amsfonts,amsthm,amssymb,bm, verbatim,dsfont,mathtools} % Use the [cmex10] option to ensure complicance
\usepackage{xspace,prettyref}
\usepackage{mathabx}  % This has widecheck  --BH

\newcommand{\reals}{{\mathbb{R}}}
\newcommand{\integers}{{\mathbb{Z}}}
\newcommand{\naturals}{{\mathbb{N}}}
\newcommand{\eexp}{{\rm e}}
\newcommand{\prob}[1]{ \mathbb{P}\left\{ #1 \right\} }
\newcommand{\expect}[1]{\mathbb{E}\left[ #1 \right]}

\newcommand{\expectLm}[1]{\mathbb{E}_{\lambda,m}\left[ #1 \right]}
\newcommand{\expectLone}[1]{\mathbb{E}_{\lambda,1}\left[ #1 \right]}

\newcommand{\norm}[1]{\left\|{#1} \right\|}

\newcommand{\var}{\mathsf{var}}
\newcommand{\Cov}{\text{Cov}}
\renewcommand{\tilde}{\widetilde}
\renewcommand{\hat}{\widehat}
\renewcommand{\check}{\widecheck}

\usepackage{color}
\usepackage{ifthen}

\newtheorem{example}{Example}

\newtheorem{proposition}{Proposition}
\newtheorem{remark}{Remark}
\newtheorem{corollary}{Corollary}
\newtheorem{lemma}{Lemma}
\newtheorem{conjecture}{Conjecture}

\newcommand{\post}[2]{\begin{center} \includegraphics[width=#2cm]{#1} \end{center} }

\renewcommand{\tilde}{\widetilde}
\renewcommand{\hat}{\widehat}
\newcommand{\calL}{{\cal L}}
\newcommand{\calO}{{\cal O}}
\newcommand{\calF}{{\cal F}}
\newcommand{\sel}{{\sf sel}}
\newcommand{\Ber}{{\sf Ber}}
\newcommand{\BA}{Barab\'{a}si-Albert }
\newcommand{\ML}{\mbox{\footnotesize \rm ML}}
\newcommand{\MAP}{\mbox{\footnotesize \rm MAP}}

\begin{document}
\pagestyle{plain}
%\title{Recovering a Hidden Community in a Preferential Attachment Graph}
\title{Community Recovery in a Preferential Attachment Graph}

\author{%
\IEEEauthorblockN{Bruce Hajek and Suryanarayana Sankagiri}
\IEEEauthorblockA{Department of Electrical and Computer Engineering \\
and the Coordinated Science Laboratory \\
University of Illinois\\
Email: \{b-hajek,ss19\}@illinois.edu}
}

\maketitle

\begin{abstract}
A message passing algorithm is derived for recovering communities within a graph generated by a variation of the Barab\'{a}si-Albert preferential attachment model.  The estimator is assumed to know the arrival times, or order of attachment, of the vertices.   The derivation of the algorithm is based on belief propagation under an independence assumption. Two precursors to the  message passing algorithm are analyzed: the first is a degree thresholding (DT) algorithm and the second is an algorithm based on the arrival times of the children (C) of a given vertex,  where the children of a given vertex are the vertices that attached to it.  Comparison of the performance of the algorithms shows it is beneficial to know the arrival times, not just the number, of the children.   The probability of correct classification of a vertex is asymptotically determined by the fraction of vertices arriving before it.   Two extensions of Algorithm C are given: the first is based on joint likelihood of the children of a fixed set of vertices; it can sometimes be used to seed the message passing algorithm.   The second is the message passing algorithm.  Simulation results are given.\footnote{This paper was presented in part at the {\em 2018
IEEE International Symposium on Information Theory}}
\end{abstract}


%\textit{A full version of this paper is accessible at arXiv}

\section{Introduction}
Community detection, a form of unsupervised learning, is the task of identifying dense subgraphs within a large graph.  For
surveys of recent work, see \cite{Fortunato10,Moore17,Abbe17}.  
Community detection is often studied in the context of a generative random graph model, of which the stochastic block model
is the most popular. The model specifies how the labels of the vertices are chosen, and how the edges are placed, given the
labels. The task of community detection then becomes an inference problem; the vertex labels are the parameters to be inferred,
and the graph structure is the data. The advantage of a generative model is that it helps in the design of algorithms for community detection.

The stochastic block model fails to capture two basic properties of networks that are seen in practice. Firstly, it does not model
networks that grow over time, such as citation networks or social networks. Secondly, it does not model graphs with heavy-tailed
degree distributions, such as the political blog network~\cite{Adamic05}.  %, have a few nodes with very high degrees (hubs). 
%Existing community detection algorithms, developed for the stochastic block model, may not perform well on real networks, even though they achieve
%remarkable performance in theory.  
The Barab\'{a}si-Albert model~\cite{barabasi1999emergence}, a.k.a.\ the preferential attachment model, is a popular
random graph model that addresses both the above shortcomings.
We use the variation of the model introduced by Jordan \cite{Jordan13} that includes community structure.
The paper \cite{Jordan13} considers labels coming  from a metric space,  though a section of  the paper focuses on the case the label
space is finite.  We consider only a finite label set--the model is described in Section \ref{sec:background}.
In recent years there has been substantial study of a variation of preferential attachment model introduced in
\cite{BianconiBarabasi01} such that different vertices can have different {\em fitness}.   For example, in a citation
network, some papers attract more
citations than others published at the same time.   There has also been work done on recovering clusters from
graphs with different fitness (see Chapter 9 of \cite{Barabasi16} and references therein).   Our work departs from previous work
by considering community detection for the model in which the affinity for attachment between an arriving
vertex and an existing vertex depends on the labels of both vertices (i.e. for the model of  \cite{Jordan13}).

 The algorithm we focus on is message passing. Algorithms that are precursors to message passing, in which the membership of
 a vertex is estimated from its radius one neighborhood in the graph, are also discussed.   The algorithm is closest in spirit to that in the papers
 \cite{Montanari:15OneComm,HajekWuXu_one_beyond_spectral15}. 
 Message passing algorithms are \emph{local} algorithms; vertices in the graph pass messages to each of their neighbors, in an iterative fashion. The messages in every iteration are computed on the basis of messages in the previous iteration.
The degree growth rates for vertices in different communities are different (unless there happens to be a tie) so
the neighborhood of a vertex conveys some information about its label. A quantitative estimate of this information is the belief (a posteriori probability) of belonging to a particular community. A much better estimate of a vertex's label could potentially be obtained if the labels of all other vertices were known. Since this information is not known,  the idea of message passing algorithms is to have
 vertices simultaneously update their beliefs.
 
 The main similarity between the preferential attachment model with communities and the stochastic block model is that
 both produce locally tree-like graphs.  However, the probabilities of edges existing are more complicated for preferential
 attachment models.   To proceed to develop the message passing algorithm, we invoke an independence assumption
 that is suggested by an analysis of the joint degree evolution of multiple vertices.    This approach is tantamount to
 constructing a belief propagation algorithm for a graphical model that captures the asymptotic distribution  of
 neighborhood structure for the preferential attachment graphs. 
 
 \paragraph{Organization of the paper}  Section  \ref{sec:preliminaries} lays the groundwork for the problem formulation
 and analysis of the community detection problem.  It begins by presenting  a model for a graph with preferential attachment
 and community structure,  following \cite{Jordan13}.   The section then presents some key properties of the graphical
 model in the limit of a large number of vertices.  In particular,  the empirical distribution of degree, and the evolution of
 degree of a finite number of vertices, are examined.  Stochastic coupling and total variation distance are used extensively.   In addition, it is shown that the growth rate parameter for a given fixed
 vertex can be consistently estimated as the size of the graph converges to infinity.
 Section  \ref{sec:recovery_from_children} formulates the community recovery problem as a Bayesian hypothesis testing problem,
 and focuses on two precursors to the message passing algorithm.   The first, Algorithm C, estimates the community membership of a vertex
based on the children of the vertex (i.e. vertices that attached to the vertex).   The second, Algorithm DT, estimates the
community membership of a vertex based on the number of children.  Section \ref{sec:Z_inference} investigates an
asymptotically equivalent recovery problem, based on a continuous-time random process $Z$ that approximates the
evolution of degree of a vertex in a large graph.  A key conclusion of that section is that, for the purpose
of estimating the community membership  of a single vertex, knowing the neighborhood of the vertex in the graph is
significantly more informative than knowing the degree of the vertex.   Section \ref{sec:perf_scaling} presents our main results
about how the performance of the recovery Algorithms C and DT scale in the large graph limit.  
Section \ref{sec:joint_estimation} presents an extension of Algorithm C whereby the labels of a fixed small set
of vertices are jointly estimated based on the likelihood of their joint children sets.   This algorithm has exponential
complexity in the number of labels estimated, but can be used to seed the message passing algorithm.  Since the vertices
that arrive early have large degree, it can greatly help to correctly estimate the labels of a small number of such vertices.
The message passing algorithm is presented  in Section \ref{sec:message_passing}.  
Simulation results are given for a variety of examples in Section  \ref{sec:simulations}.
Various  proofs, and the derivation of the message passing
algorithm,  can be found in the appendices.
 
\paragraph{Related work}
 %%  SOME RELATED WORK ON MODEL
  A different extension of preferential attachment to include communities is given in
 \cite{AntunovicMosselRacz16}.  In  \cite{AntunovicMosselRacz16}, the community membership of a new
vertex is determined based on the membership of the vertices to which the new vertex is attached.
The paper focuses on the question of whether large communities coexist as the number of vertices converges
to infinity.  However, the dynamics of the graph itself is the same as in the original Barab\'{a}si-Albert model.
In contrast, our model assumes that community membership of a vertex is determined randomly
before the vertex arrives, and the distribution of attachments made depends on the community membership.
It might be interesting to consider a combination of the two models, in which some vertices determine
community membership exogenously, and others determine membership based on the memberships
of their neighbors.

%% SOME RELATED WORK ON RECOVERY WITH DIFFERENT MODEL
 Another model of graphs with community structure and possibly heavy-tailed degree distribution is the
 degree corrected stochastic block model -- see \cite{ChenLiXu15} for recent work and references.

%% SOMETHING ABOUT PREFERENTIAL ATTACHMENT IN GENERAL
 There is an extensive literature on degree distributions and related properties of preferential attachment graphs,
and an even larger literature on the closely related theory of Polya urn schemes.  However,  the
addition of planted community structure breaks the elegant exact analysis methods, such as the matching equivalence
formulated in \cite{bollobas2001degree}, or methods such as in \cite{Janson06} or \cite{PekozRossRollin14}.
Still, the convergence of the empirical distribution of the induced labels of half edges (see Proposition  \ref{prop:empirical_m1} below)
makes the analysis tractable without the exact formulas.   A sequence of models evolved from preferential attachment
with fitness \cite{BianconiBarabasi01},  towards the case examined in \cite{Jordan13}, such that the attachment probability
is weighted by a factor depending on the labels of both the new vertex and a potential target vertex.  The model of
\cite{FlaxmanFriezeVera06} is a special case, for which attachment is possible if the labels are sufficiently close.
See \cite{Jordan13,FlaxmanFriezeVera06,Barabasi16} for additional background literature.

\section{Preliminaries and some asymptotics}   \label{sec:preliminaries}
 
\subsection{Barab\'{a}si - Albert preferential attachment model with community structure}  \label{sec:background}

The model consists of  a sequence of directed graphs, $(G_t=(V_t, E_t): t\geq  t_o)$ and
vertex labels $(\ell_t : t\geq 1)$  with distribution determined by the following parameters:\footnote{The
model is the same as the finite metric space case of \cite{Jordan13} except for differences in notation.
$\alpha, S, X, \mu, \nu, Y, \phi$ in \cite{Jordan13} are $\beta^T, [r],  \ell, \rho, \eta, C, 2\theta$ here.  Also,
\cite{Jordan13} denotes the initial graph as $G_0$ while we denote it by $G_{t_o}$, we assume it has
$mt_0$ edges, and we  suppose the random evolution begins with the addition of vertex $t_o+1.$}
\begin{itemize}
\item $m\geq 1$ :  out degree of each added vertex
\item $r\geq 1$:   number of possible labels; labels are selected from $[r]\triangleq \{1,\ldots , r\}$
\item  $\rho = (\rho_1, \ldots , \rho_{r}) $:  a priori label probability distribution
\item $\pmb\beta \in \reals^{r\times r}$:  matrix of strictly positive affinities for vertices of different labels; $\beta_{uv}$ is the affinity of
a new vertex with label $u$ for attachment to a vertex of label $v.$
\item $t_o \geq 1$: initial time
\item  $G_{t_o}= (V_{t_o}, E_{t_o})$: initial directed graph with $V_{t_o}=[t_o]$ and $mt_o$ directed edges
\item $(\ell_t : t \in [t_o]) \in [r]^{t_o}$: labels assigned to vertices in $G_{t_o}.$
\end{itemize}
For each $t\geq t_o$, $G_t$  has $t$ vertices given by $V_t=[t]$ and $mt$ edges.
The graphs can contain parallel edges.  No self loops are added during the evolution, so if $G_{t_o}$ has no self loops,
none of the graphs will have self loops.  Of course, by ignoring the orientation of edges, we  could obtain undirected graphs.
% For example, we could take $t_o=1$ with $m$ edges forming self loops to vertex 1.  Or we could take $t_o=2$ with $m$ edges
%directed from vertex 1 to vertex 2 and $m$ edges directed from vertex 2 to vertex 1. 

Given the labeled graph $G_t,$  the graph $G_{t+1}$ is constructed as follows.   First vertex $t+1$ is
added and its label $\ell_{t+1}$ is randomly selected from $[r]$ using distribution $\rho,$  independently of $G_t.$
Then $m$ outgoing edges are attached to the new vertex, and the head ends of those edges are selected from
among the vertices in $V_t=[t]$ using sampling with replacement, and probability distribution given by preferential
attachment, weighted based on labels according to the affinity matrix.

The probabilities are calculated as follows.   Note that $E_t$ has $m t$ edges, and thus $2mt$ {\em half edges}, where
we view each edge as  the union of two half edges.  For any edge, its two half edges are each incident to a vertex;
 the vertices the two half edges are incident to are the two vertices the edge is incident to.
Suppose each half edge inherits the label from the vertex it is incident to.
 If  $\ell_{t+1}=u$, meaning the new vertex has label $u,$  and if one of the existing half edges has label $v,$  then
the half edge is assigned weight $\beta_{uv}$ for the purpose of adding edges outgoing from vertex $t+1.$     For each one of the new edges
outgoing from vertex $t+1,$ an existing half edge is chosen at random from among the $2mt$ possibilities, with probabilities
proportional to such weights.    The selection is done simultaneously for all $m$ of the new edges, or equivalently, sampling
with replacement is  used.   Then the vertices of the respective selected half edges become the head ends of the  $m$ new edges.

\subsection{Empirical degree distribution for large $T$} \label{sec:emp_degree_dist}

For a vertex in $G_t$, where $t\geq t_o,$  the distribution of the number of
edges incident on the vertex from vertex $t+1$ depends on the label of the vertex,
the degree of the vertex,  and the labels on all the half edges incident to the existing vertices in $G_t.$
The empirical distribution of labels of half edges in $G_t$ converges almost surely as $t \to \infty,$
as explained next.   Let $C_t=(C_{t,u}: u \in [r] )$  for $t\geq t_o$,  where $C_{t,u}$ denotes
the number of half edges with label $u$ in $G_t.$
It is easy to see that $(C_t: t\geq t_o)$ is a discrete-time Markov process, with initial state determined
by the labels of vertices in $G_{t_o}.$
 Let $\eta_t = \frac{C_t}{2mt}.$
Thus, $\eta_{t, u}$ is the fraction of half edges that have label $u$  at time $t.$
%We shall view $C_t$ and $\eta_t$ as row vectors. 
Let $h=(h_1, \ldots , h_r)$ where
\begin{align}   \label{eq:h_def_general}
h_v(\eta)  =   \rho_v   +   \sum_u   \rho_u \left(  \frac{\beta_{uv}\eta_{v}}{\sum_{v'} \beta_{uv'}  \eta_{v'} } \right)  - 2\eta_v .
\end{align}
The following is proved in \cite{Jordan13}, by appealing to the theory of stochastic approximation.
For convenience we give essentially the same proof, using our notation, in Appendix \ref{sec:global_convergence}.
\begin{proposition}  \cite{Jordan13} \label{prop:gobal_convergence}
(Limiting fractions of half edges with given labels)
$\eta_t \to \eta^*$ a.s. as $t\to \infty,$  where $\eta^*$ is the unique probability vector such that $h(\eta^*)=0.$
\end{proposition}

A second limit result we restate from \cite{Jordan13} concerns the empirical degree distribution
for the vertices with a given label.    For  $v\in [r]$ and integers $n\geq m$ and $T,$ let:
\begin{itemize}
\item  $H^v(T)$ denote the number of vertices with label $v$ in $G_T$
\item  $N_n^v(T)$ denote the number of vertices with label $v$ and with degree $n$ in $G_T$
\item   $P_n^v(T)  = \frac{N_n^v(T)}{H^v(T)}$ denote the fraction of vertices  with label $v$ that have degree $n$ in $G_T.$
\end{itemize}
Let
\begin{align*} % \label{eq:theta_uv:definition}
\theta_{u,v}^* =    \frac{\beta_{uv}}{2\sum_{v'} \beta_{uv'}\eta^*_{v'}}    ~~\mbox{for}~u,v \in [r],
 \end{align*}
 and
 \begin{align}  \label{theta_v_def}
\theta^*_v = \sum_u \rho_u \theta_{u,v}^*  ~~\mbox{for}~v \in [r].
\end{align}
\begin{proposition}   \label{prop:empirical_m1} \cite{Jordan13}
(Limiting empirical distribution of degree for a given label) 
Let $n\geq m$ and $v\in [r]$  be fixed. 
Then $\lim_{T \to \infty}  P_n^v(T)   = p_n(\theta^*_v,m)$ almost surely,  where
\begin{align}
p_n(\theta, m)  & 
= \frac{  \Gamma\left(\frac 1 \theta  + m  \right)  \Gamma(n)  }    {\theta \Gamma(m) \Gamma\left( n+ \frac 1 \theta + 1   \right)  } \nonumber   \\
&   \asymp\left[ \frac {\Gamma\left(\frac 1 \theta  + m  \right)} {\theta \Gamma(m) }  \right]  
\frac 1 { n^{\frac 1 \theta +1 }  }      \label{eq:tail_asymp_m1}
\end{align}
\end{proposition}
The asymptotic equivalence in \eqref{eq:tail_asymp_m1} as $n\to\infty$ follows from
Sterling's formula for the Gamma function.
The proposition shows that the limiting degree distribution of a vertex with label $v$ selected uniformly at
random from among the vertices with label $v$ in $G_T$ has probability mass function with tail decreasing
like $n^{-\left(\frac 1 {\theta_v^*} +1\right) }.$  If $\beta_{u,v}$ is the same for all $u,v$ then $\theta^*_v=1/2$ for all $v$
and we see the classical tail exponent -3 for the Barab\'{a}si-Albert model.

The proof of Proposition \ref{prop:empirical_m1} given in \cite{Jordan13} is based on examining
the evolution of the fraction of vertices with a given label and given degree $n.$    Using the convergence
analysis of stochastic approximation theory, this yields limiting difference equations for $p_n$ that can be
solved to find $p_n.$  However, since all vertices with a given label are grouped together, the analysis does not
identify the limiting degree distribution of a vertex as a function of the arrival time of the vertex.

The following section investigates the evolution of the degree of a single vertex, or finite set of vertices,
conditioned on their labels.   As a preliminary application, we produce an alternative proof of
Proposition \ref{prop:empirical_m1} in Appendix \ref{app:empirical_m1}.  The main motivation for this
alternative approach is that it  can also be applied to analyze the probability of label error as a function of
time of arrival of a vertex,  for two of the recovery algorithms we consider.

%%%%%%%%%%%%%%%%%%%%%%%%%%%%%%%%%%
%\section{Evolution of vertex degree} 
%\label{sec:vertex_deg_evolution}

\subsection{Evolution of vertex degree--the processes $Y, \tilde Y, \check Y,$ and $Z$}   \label{sec:tildeY_def}

Consider the preferential attachment model defined in
Section \ref{sec:background}.  Given a vertex $\tau$ with
$\tau \geq t_o+1$, consider the process $(Y_t : t\geq \tau),$ where
$Y_t$ is the degree of vertex $\tau$ at time $t.$
So $Y_{\tau} = m.$ 
%and,  for convenience, we set  $Y_t=0$ for $1 \leq t < \tau,$  indicating that vertex $\tau$ is not active before time $\tau.$
The conditional distribution (i.e. probability law) of
$Y_{t+1}-Y_t$  given $(Y_t, \eta_t, \ell_\tau=v, \ell_{t+1}=u)$ is
given by:
$$
\calL (Y_{t+1} -Y_t   |  Y_t, \eta_t, \ell_\tau=v, \ell_{t+1}=u)  = 
 {\sf binom} \left( m , \frac{\theta_{u, v, t} Y_t}{mt} \right),
$$
where
\begin{align*}
 \theta_{u,v,t }=  \frac{\beta_{uv}}{2 \sum_{v'} \beta_{uv'}\eta_{tv'}}.        
\end{align*}
It follows that, given  $(Y_t, \eta_t, \ell_\tau=v),$  the conditional distribution of $Y_{t+1}-Y_t$ is a mixture of binomial
distributions with selection probability distribution $\rho$, which we write as:
\begin{align*}
\calL (Y_{t+1} -Y_t   |  Y_t, \eta_t, \ell_\tau=v)   = \sum_{u\in [r] }  \rho_u  {\sf binom} \left( m , \frac{\theta_{u, v, t} Y_t}{mt} \right).
\end{align*}
Proposition \ref{prop:gobal_convergence} implies,
  given any $\epsilon > 0$,  if $\tau$ is sufficiently
large,  $\prob{ \norm{\eta_t -\eta^*} \leq \epsilon \mbox{ for all } t\geq \tau} \geq 1-\epsilon.$
Therefore, $\theta_{u,v.t} \approx \theta_{u,v}^*$ for $v\in [r].$
A mixture of binomial distributions, all with small means, can be well
approximated by a Bernoulli  distribution with the same mean.  Thus, we expect
 $\calL(Y_{t+1}-Y_t | Y_t,\ell_\tau=v) \approx  \Ber \left(  \frac{ \theta^*_vY_t} t\right).$

Based on these observations, we define a random process that is an idealized variation of $Y$ obtained
by replacing $\eta_t$ by the constant vector $\eta^*,$ and allowing jumps of size one only.
The  process $\tilde Y$ has parameters
$\tau, m,$ and $\vartheta$,  where $\tau$ is the  activation time, $m$ is the state at the
activation time, and $\vartheta > 0$  is a rate parameter.   The process $\tilde Y$ is a time-inhomogeneous
Markov process %and it is deterministic up til time  $\tau.$   Specifically, $\tilde Y_t = 0$ for $1 \leq t \leq \tau-1$ and
with initial value $Y_\tau=m.$ 
For $t\geq \tau$ and $y$ such that $\frac{\vartheta y}{t} \leq 1$, we require:
\begin{align}  
\calL ( \tilde Y_{t+1}-\tilde Y_t   |  \tilde Y_t =y  )  = \Ber\left(  \frac{ \vartheta y}{t} \right).  \label{eq:tilde_Y_law}
\end{align}
By induction, starting at time $\tau$, we find that $\tilde Y_t \leq m + t-\tau$ for $t\geq \tau.$
If $\tau \geq m$ and $\vartheta \leq 1$, then $\frac{ \vartheta \tilde Y_t}{t} \leq 1$ for
all $t\geq \tau$ with probability one,  in which case \eqref{eq:tilde_Y_law} and
the initial condition completely specify
the distribution of $(\tilde Y : t\geq \tau).$
However, for added generality we allow $\vartheta > 1,$ in which case the above construction
can break down.   To address such situation, we define $\zeta$ such that $\zeta $ is the stopping time
$\zeta  \triangleq  \inf\{ t :  \vartheta \tilde Y_t  > t \}$ and we define
$\tilde Y_t = + \infty$ for $t  >  \zeta.$  

The process $Y$ can be thought of as a (non Markovian) discrete time birth process with activation time
$\tau$ and birth probability at a time $t$ proportional to the number of individuals.   However, the birth probability (or birth rate)
per individual, $\frac {\theta^*_{v}} t$, has a factor $\frac 1 t,$ which tends to decrease the birth rate per individual.
To obtain a process with constant
birth rate per individual we introduce a time change by using the process $(Y_{e^s}: s\geq 0).$  In other words, we
use $t$ for the original time variable and $s=\ln t$ as a new time variable.
We will define a process $Z$ such that $(Z_{\ln (t/ \tau)} : t\geq \tau)  \approx  (Y_t: t\geq \tau)$,  
or equivalently,  $(Z_s: s\geq 0)  \approx  (Y_{\tau \eexp^s}: s  \geq 0),$ in a sense to be made precise.

The process $Z=(Z_s: s\geq 0)$ is a continuous time pure birth Markov process with initial
state $Z_0=m$ and birth rate $\vartheta k$ in state $k,$ for some $\vartheta > 0.$
(It is a simple example of a Bellman-Harris process, and is related to busy periods in Markov queueing systems.)
% Also, let $Z_s= 0$ for $s<0.$    
The process $Z$ represents the total number of individuals in a continuous
time branching process beginning with $m$ individuals activated at time 0,
such that each individual spawns another at rate $\vartheta.$ 
For fixed $s$, $Z_s$ has the negative binomial distribution ${\sf negbinom}(m, \eexp^{-s\vartheta}).$
In other words,  its marginal probability distribution
$(\pi_n(s,\vartheta,m):   n  \in \integers_+)$ is given by
\begin{align}   \label{eq:Zdist}
\pi_n(s,\vartheta,m)  = \binom{n-1}{m-1}  \eexp^{-m\vartheta s} (1 - \eexp^{-\vartheta s} )^{n-m}~~~\mbox{for}~n\geq m.
\end{align}
In particular, taking $m=1$ shows $\pi(s,\vartheta,1)$ is the geometric distribution with parameter  $\eexp^{-\vartheta s}$, and hence,
mean  $\eexp^{\vartheta s}.$     The expression \eqref{eq:Zdist} can be easily derived for $m=1$ by solving the Kolmogorov forward equations
recursively in $n$:   $\dot \pi_n = -\vartheta n  \pi_n  + \vartheta (n-1)\pi_{n-1}$ for $n\geq 1,$ with the convention and base case,
$\pi_0 \equiv 0.$     For $m \geq 2,$  the process $Z$ has the same distribution as the sum of $m$ independent copies of $Z$ with $m=1,$
proving the validity of  \eqref{eq:Zdist} by the same property for the negative binomial distribution.

Let $\check Y_t = Z_{\ln (t /\tau)}$ for integers $t \geq \tau.$   The mapping from $Z$ to $\check Y$ does not depend
on the parameter $\vartheta,$ so a hypothesis testing problem for $Z$ maps to a hypothesis testing problem for $\check Y.$
There is loss of information because the mapping is not invertible, but the loss tends to zero as $\tau \to\infty,$  because the rate
of sampling of $Z$ increases without bound.

The following proposition, proven in Appendix \ref{app:proof_of_Y_couple}, shows that $Y, \tilde Y$ and $\check Y$ are asymptotically equivalent in the sense of total variation distance.
Since the processes $Y, \tilde Y$ and $\check Y$ are integer valued, discrete time processes,  their trajectories
over a finite time interval $[\tau,T]$ have discrete probability distributions.  See the beginning
of  Appendix \ref{app:proof_of_Y_couple} for a review of the definition of total variation distance and its significance for coupling.
Sometimes we write $\tilde Y(\vartheta)$ instead of $\tilde Y$,  and $\check Y(\vartheta)$ instead of $\check Y,$
to denote the dependence on the parameter $\vartheta.$

\begin{proposition}   \label{prop:Y_couple}
Suppose $\tau, T \to \infty$ such that $T > \tau$ and $T /\tau$ is bounded. Fix $v\in [r].$
Then
\begin{align}   \label{eq:Y_tildeY_couple}
d_{TV}(  (Y_{[\tau,T]} | \ell_{\tau}=v) ,  \tilde Y_{[\tau,T]}(\theta^*_v) )  \to 0,
\end{align}
and for any $\vartheta > 0$,
\begin{align}    \label{eq:tildeY_checkY_couple}
d_{TV}\left(\tilde Y_{[\tau, T]}(\vartheta)   , \check Y_{[\tau, T]}(\vartheta)   \right) \to 0.
\end{align}
\end{proposition}

The first part of Proposition \ref{prop:Y_couple} can be strengthened as follows.    The labels in
$\ell_{[1,T]}$ are mutually independent, each with distribution $\rho.$  We
can define a joint probability distribution over $(\tilde Y_{[\tau, T]}, \ell_{[1,T]})$ by specifying
the conditional probability distribution of $\tilde Y_{[\tau, T]}$ given $\ell_{[1,T]}$ as follows.
Given $\ell_{[1,T]}$,  $\tilde Y_{[\tau, T]}$ is a Markov sequence with $\tilde Y_{\tau}=m$ and:
\begin{align}   \label{eq:tildeY_given_labels}
\calL (\tilde Y_{t}- \tilde Y_{t-1}  \big| \ell_{t}=u,\ell_\tau=v, \tilde Y_{t-1}=y) =  \mathsf{Ber}\left(\frac{ y  \theta_{u,v}^*} {t-1} \right).
\end{align}
By the law of total probability, this gives the same marginal distribution for $\calL ( \tilde Y_{[\tau, T]} | \ell_{\tau = v} )$ as
\eqref{eq:tilde_Y_law} with $\vartheta = \theta_v^*,$ as long as  $ \max_{u,v} \{\theta^*_{u,v}\} y \leq t.$
%Note that, by Bayes formula,
%\begin{align}   \label{eq:Bayes_tildeY}
%\prob{\ell_{t} = u | \tilde Y_{[\tau, T]}, \tilde Y_{t}- \tilde Y_{t-1} = 1, \ell_{\tau}=v } = \frac{\theta^*_{u,v}}{\theta^*_v}  
%\end{align}
\begin{proposition}   \label{prop:Y_ell_couple}
Suppose $\tau, T \to \infty$ such that $T > \tau$ and $T /\tau$ is bounded. Fix $v\in [r].$
Then
\begin{align}   \label{eq:Y_tildeY_coupled}
d_{TV}\left(  \left(Y_{[\tau,T]}, \ell_{[1,T]} \right)  ,  \left(\tilde Y_{[\tau,T]},  \ell_{[1,T]}\right) \right)  \to 0,
\end{align}
\end{proposition}
The proof is a minor variation of the proof
of  Proposition \ref{prop:Y_couple}
because the estimates on total variation distance
are uniform for $\theta^*_v$ or $\theta^*_{u,v}$
bounded.  Details are left to the reader.

\subsection{Joint evolution of vertex degrees}
Instead of considering the evolution of degree of a single vertex we consider the evolution of degree
for a finite set of vertices,  still
for the preferential attachment model with communities, $(G_t = (V_t, E_t): t \geq t_o),$
defined in Section \ref{sec:background}.
% with $m\geq 1$  and parameters $r, \rho,  \beta , t_o, G_{t_o}$,  and $(\ell_t : t \in [t_o]) \in [r]^{t_o}.$
Given integers $\tau_1, \ldots  , \tau_J$ with $t_o <  \tau_1 < \cdots <  \tau_J$, let
$Y^j_t=0$ if $1\leq t < \tau_j$ and let $Y^j_t$ denote the degree of vertex $\tau_j$ at time $t$ if $t\geq \tau_j$.
Let $Y^{[J]}_t = (Y^j_t : j \in [J]).$
Let $(v_1, \ldots , v_J) \in [r]^J.$
We consider the evolution of $(Y^{[J]}_t : t\geq 1)$ given
$(\ell_{\tau_1}, \ldots  , \ell_{\tau_J}) = (v_1, \ldots  , v_J).$
Let $\vartheta_j = \theta^*_{v_j}$ for $j \in [J].$
About the notation $\theta^*$ vs. $\vartheta$:  The vector $\theta^* = (\theta_v^* : v\in [r])$ denotes the limiting rate
parameters for the $r$ possible vertex labels defined in \eqref{theta_v_def}, whereas  $\vartheta = (\vartheta_j :  j \in [J])$ denotes the limiting
rate parameters  for the specific set of $J$ vertices being focused on, conditioned on their labels being $v_1, \ldots  , v_J.$

The process  $\tilde Y^{[J]}$ is defined similarly.
Fix  $J\geq 1$, integers $\tau_1, \ldots , \tau_J$  with
$1 \leq  \tau_1 < \ldots   < \tau_J$,  and $\underline \vartheta \in (\reals_{>0})^J.$
Suppose for each $j\in [J]$ that $\tilde Y^j$ is a  version of the process
$\tilde Y$ defined in  Section \ref{sec:tildeY_def},  with parameters
$\tau_j$, $m$, and $\vartheta_j,$   with the extension $\tilde Y^j_t =0$ for $1 \leq t \leq \tau_j-1.$
Furthermore,  suppose the $J$ processes  $(\tilde Y^j)_{j\in [J]}$ are mutually independent.
Finally, let $\tilde Y^{[J]} = (\tilde Y^{[J]}_t :  t\geq 1 )$  where $\tilde Y^{[J]}_t  = (\tilde Y^{j}_t : j\in [J]).$
Note that $\tilde Y^{[J]}$ is itself a time-inhomogeneous Markov process.
In what follows we write $\tilde Y^{[J]}(\underline \vartheta)$ instead of
$\tilde Y^{[J]}$ when we wish to emphasize the dependence  on the parameter
vector $\underline \vartheta.$  Let $\check Y^{[J]}$ be defined analogously, based on $\check Y.$

\begin{proposition}   \label{prop:YJ_couple}
Fix the parameters of the preferential attachment model,  $m,r, \beta, \rho, t_o, G_{t_o}, \ell_{[1,t_o]}.$
Fix $J\geq 1$ and $v_1, \ldots  ,  v_J \in [r],$ and let $\vartheta_j = \theta^*_{v_j}$ for $j\in [J].$
Let $\tau_0 \to \infty$ and let $\tau_1, \ldots , \tau_J$ and $T$ vary such that
 $\tau_0  \leq  \tau_1 < \ldots   < \tau_J$,  and $T/\tau_0$ is bounded.
Then
\begin{align*}
&d_{TV}\left( \tilde Y^{[J]}_{[1,T]}(\underline \vartheta) , \left(Y^{[J]}_{[1,T]}  \big| \ell_{\tau_j}
 = v_j ~\mbox{\rm for } ~ j\in [J]\right)  \right) \to 0   \\
& d_{TV}\left( \tilde Y^{[J]}_{[1,T]}(\underline \vartheta) , 
                     \check Y^{[J]}_{[1,T]}(\underline \vartheta)  \right)  \to 0  
\end{align*}
\end{proposition}

The proposition is proved in Appendix \ref{app:proof_YJ_couple}.
A key implication of the proposition is that the degree evolution processes for a finite number
of vertices are asymptotically independent in the assumed asymptotic regime.  In particular, the following corollary
is an immediate consequence of the proposition.   It shows that the degrees of $J$
vertices at a fixed time $T$ are asymptotically independent with marginal
distributions given by \eqref{eq:Zdist}.
\begin{corollary}   \label{cor:Y_vs_Z_mult} (Convergence of joint distribution of degrees of $J$ vertices at a given time)
Under the conditions of Proposition \ref{prop:YJ_couple},  for a vector
$\underline n =(n_1, \ldots  , n_J)$ with $n_j \geq m,$
\begin{align*}
\lim_{\tau_0 \to\infty}  \sup_{\tau_1, \ldots  , \tau_J, T}  \bigg|   \prob{ Y^{[J]}_T    =  \underline n 
\bigg|  (\ell_{\tau_1}, \ldots  , \ell_{\tau_J}) = (v_1, \ldots  , v_J)   }  \\
  -  \prod_{j\in [J]}     \pi_{n_j}\left(\ln(T/\tau_j),   \vartheta_j,  m\right)    \bigg| = 0 .~~~~~~~~~~~~~~~~~~~~~~~~~
\end{align*}
\end{corollary}

\begin{remark}
Corollary \ref{cor:Y_vs_Z_mult} implies, given $\ell_\tau = v$, the limiting distribution of the
degree of $\tau$ in $G_T$ is ${\sf negbinom}(m, (\tau/T)^{\theta^*_v}),$  as  $\tau, T \to \infty$
with  $\tau \leq T$ and $T/\tau$ bounded.    This generalizes the result known in the classical case
$\beta_{u,v}\equiv 1$ where $\theta_v^* = 1/2$, shown on p. 286 of  \cite{bollobas2001degree}.
\end{remark}

\subsection{Large time evolution of degree of a fixed vertex and consistent
estimation of the rate parameter of a vertex}

Consider the \BA model with communities.  Fix $\tau \geq 1$ and let
$Y_t$ denote the degree of $\tau$ in $G_t$ for $t\geq t_o.$
To avoid triviality, assume $\tau$ is not an isolated vertex in
the initial graph $G_{t_o}.$ 
The following proposition offers a way to consistently estimate
the rate parameter $\theta^*_{\ell_\tau}.$  If the
parameters $\theta^*_v$ of the \BA model are distinct, it follows that any fixed
finite set of vertices could be consistently classified in the limit as $T\to\infty$, without
knowledge of the model parameters.
\begin{proposition}  (Large time behavior of degree evolution)
  \label{prop:Y_large_time} 
  For $\tau$ fixed,
\begin{align}  \label{eq:Y_long_term}
\lim_{T\to\infty} \frac{\ln Y_T}{\ln(T/\tau) } = \theta^*_{\ell_{\tau}} ~~  a.s.
\end{align}
Here, ``a.s." means almost surely, or in other words, with probability one.
\end{proposition}
The following strengthening of Proposition \ref{prop:Y_large_time} is conjectured.
\begin{conjecture}  (Sharp large time behavior of degree evolution) 
\label{conj:Y_large_time_sharp}
For $\tau$ fixed,
\begin{align}  \label{eq:Y_conjecture}
 \lim_{T\to \infty}  \frac{Y_T}{(T / \tau)^{ \theta^*_{\ell_{\tau}} }}= W   ~  a.s.  
\end{align}
for a random variable $W$ with $\prob{W>0}=1.$
\end{conjecture}
See Appendix \ref{app:consistency} for a proof of the proposition and a
proof that \eqref{eq:Y_conjecture} holds with $Y$ replaced by $\check Y.$

\section{Community recovery based on children}  \label{sec:recovery_from_children}

 Given vertices $\tau$ and $\tau_0$, we say
$\tau$ is a child of $\tau_0$, and $\tau_0$ is a parent of $\tau$, if $\tau \geq \max\{\tau_0, t_o\}+1,$
and there is an edge from $\tau$ to $\tau_0.$    It is assumed that the
known initial graph $G_{t_o}$ is arbitrary and carries no information
about vertex labels.   Thus,  for the purpose of inferring the vertex labels,  the edges
in $G_{t_o}$ are not relevant beyond the degrees that they
imply for the vertices in $G_{t_o}.$
Assuming $T$ is an integer with  $1 \leq \tau < T$, 
 let $\partial \tau$ denote the children of $\tau$ in $G_T$ and $\wp \tau$
 the parents of $\tau.$    So  $\wp \tau = \emptyset$ if $\tau \leq t_o$ and
 $\partial \tau \subset \{t_o+1,  \ldots, T\}.$
 
Consider the problem of estimating $\ell_\tau$ given observation
of a random object $\calO.$   For instance, the object could be the degree of vertex $\tau$ in
$G_T$, or it could be the set of children of $\tau$ in $G_T,$ or it could be the entire graph.
This is an $r$-ary hypothesis testing problem.
It is assumed a priori that the label $\ell_\tau$ has probability distribution $\rho,$
so it makes sense to try to minimize the probability of error in the Bayesian framework.
Let $\Lambda_{\tau}$ denote the log-likelihood vector defined by
$\Lambda_{\tau}(\calO | i ) =  \ln p(\calO | \ell_{\tau} = i)$ for $i \in [r].$
By a central result of Bayesian decision theory, the optimal decision rule is the MAP estimator,
given by
\begin{align*}
\hat \ell_{\tau,MAP} = \arg\max_i \left(  \ln \rho_i  +  \Lambda_{\tau}(\calO | i )   \right)
 \end{align*}

\begin{remark}
(i)  Knowing $G_T$ is equivalent to knowing the indices of the vertices and the undirected graph
induced by dropping the orientations of the edges of $G_T.$ 

(ii)  The estimators considered in this paper are assumed to know the order of arrival of the vertices
(which we take to be specified by the indices of the vertices for brevity) and the parameters $m$, $\beta$ and $\rho.$
It is clear that in some cases the parameters can be estimated from a realization of the graph for sufficiently
large $T.$   In particular, the parameter $m$ is directly observable.     By Proposition \ref{prop:Y_large_time},
if the order of arrival is known,  the set of growth rates $\{\theta^*_v: v \in [r]\}$ can be estimated. So if
the $\theta^*_v$'s are distinct,  the distribution $\rho$ can also be consistently estimated.

(iii) If the indices of the vertices are not known and only the undirected version of the
graph is given, it may be possible to estimate the indices if $m$ is sufficiently large.
Such problem has been explored recently for the classical
Barab\'{a}si-Albert model \cite{LuczakMagnerSzpankowski16}, but we don't pursue it here
for the variation with a planted community.
\end{remark}

 \paragraph*{\bf Algorithm C}
 The first recovery algorithm we describe, Algorithm C (``C" for ``children"),  is to let $\calO$ denote the set of
 children, $\partial t = \{ t_1, \ldots,   t_n  \}$,
of  vertex $\tau$ in $G_T.$   Equivalently, $\calO$ could be observation of $Y_{[\tau\vee t_o,T]},$
with parameters $m$ and $\theta^*_v,$
where $\tau \vee t_o = \max\{ \tau, t_o\}.$
However, motivated
 by Proposition \ref{prop:Y_couple},  we consider instead observation of $\tilde Y_{[\tau\vee t_o,T]},$  which has a distribution
asymptotically equivalent to the distribution of $Y_{[\tau\vee t_o, T]}.$
%%%
%%%  Initial degree defined
Let $d_0(\tau)$ denote the initial degree of vertex $\tau$,
defined to be the degree of $\tau$ in $G_{t_o}$  if $\tau \leq t_o$ and $d_0(\tau)=m$ otherwise.
Given a possible children set
$\partial t = \{ t_1, \ldots,   t_n  \},$  let $y^{\partial \tau}_{[\tau,T]},$
denote the corresponding degree evolution sample path:  $y^{\partial \tau}_t = d_o(\tau) +|\partial \tau \cap [\tau,t]|$
for $\tau\vee t_o  \leq t \leq T,$ 
The probability $\tilde Y_{[\tau\vee t_o, T]}$ corresponds to
children set $\partial t = \{ t_1, \ldots,   t_n  \}$ is given by
\begin{align*}
&P( \partial t = \{ t_1, \ldots,   t_n  \} ) = \\
&~~~\prod_{t\in [\tau\vee t_o +1, T] \backslash \partial \tau }
\left( 1 - \frac{y^{\partial \tau}_{so t-1}  \theta^*_v} {t-1} \right)  \prod_{t \in \partial \tau}  \frac{y^{\partial \tau}_{t -1} \theta^*_v} {t -1},
\end{align*}
so the log likelihood for observation $\tilde Y_{[\tau\vee t_o,T]} = y^{\partial  \tau}_{[\tau\vee t_o,T]}$  is:
\begin{align*}
 \Lambda_\tau^C &  =  |\partial \tau|    \ln  \theta_v^*  +  \sum_{t\in [\tau\vee t_o +1, T] \backslash \partial \tau }
\ln   \left( 1 - \frac{y^{\partial \tau}_{t-1} \theta^*_v} {t-1} \right)  
\end{align*}
Algorithm C for estimating $\ell_\tau$ is to use the MAP estimator
based on  $\rho$ and  $\Lambda^C_{\tau}.$
Using the approximation $\ln(1+s) = s$ and approximating the sum by
an integral we find $\Lambda^C_{\tau} \approx  \lambda_{\tau},$  where
\begin{align}
&\lambda^C_{\tau}(v) \triangleq   |\partial \tau|  \ln \theta_v^*  -
      \theta_v^*   \int_{\tau\vee t_o}^T  \frac{y^{\partial \tau}_t }{t} dt  \nonumber \\
&= |\partial \tau|    \ln  \theta_v^*  + \theta_v^*   \left(d_o(\tau) \ln \frac{\tau\vee t_o}{T} + \sum_{t\in \partial \tau} 
 \ln \frac t T \right).  \label{eq:LambdaC_approx}
\end{align}

\paragraph*{\bf Algorithm DT}
The second recovery algorithm we describe, Algorithm DT  (``DT" for ``degree thresholding"), is to let $\calO$ denote the
number of children of vertex $\tau$ in $G_T$, or, equivalently, the degree of $\tau$ at time $T$ minus the initial degree of $\tau.$  
Equivalently, $\calO$ could be observation of $Y_T - d_o(\tau).$    However, motivated
by Proposition  \ref{prop:Y_couple},  we consider instead consider observation of $\check Y_T - d_0(\tau),$  which has
the ${\sf negbinom}\left(d_o(\tau), (\tau/T)^{\theta_v^*}\right)$ distribution given $\ell_\tau = v$,  for $v\in [r].$
The log likelihood vector in this case, given the number of children,  $|\partial \tau|$,  is:
\begin{align*}
\Lambda^{DT}_\tau(v)  = - d_o(\tau) \theta_v^*\ln(T/\tau) +  |\partial \tau|   \ln
 \left(  1 - (\tau/T)^{\theta_v^*}  \right) ,
\end{align*}
where we have dropped a term (log of binomial coefficient)  not depending on $v.$
Algorithm DT for estimating $\ell_\tau$ is to use the MAP decision rule
based on $\rho$ and $\Lambda^{DT},$ or in other words, the MAP decision rule based
on $\calO = \check Y_T,$ or equivalently, based on $\calO = Z_{\bar s},$
where $\bar s = \ln(T/\tau)$  (because $\check Y_T = Z_{\bar s}$). 
Let $f_Z^{DT}(\rho, \theta^*, m, \bar s)$
denote the resulting average error probability $p_e.$

\section{Hypothesis testing for $Z$}   \label{sec:Z_inference}

Proposition \ref{prop:Y_couple} gives an asymptotic equivalence of
$Y_{[\tau,T]}, \tilde Y_{[\tau,T]},$ and  $\check Y_{[\tau,T]}.$    Recall that
$\check Y_{[\tau,T]}$ is obtained by sampling the continuous time
process $Z_{\ln(t/\tau)}$ at integers $t \in [\tau, T].$   Thus, the
continuous time process $Z$ is not observable.   However, as $\tau\to\infty,$
the rate that $Z$ is sampled increases without bound, so asymptotically
 $Z_{[0, \ln (T/ \tau)]}$ is observed.   We consider here the hypothesis
 testing problem based on observation of $Z_{[0, \ln (T/ \tau)]}$ such that under $H_v$ it
has rate parameter $\vartheta  = \theta_v^*$ for $v\in [r].$   This is sensible
in case the parameter values $\theta^*_v$, $v\in [r],$  are distinct.   To this
end, we derive the log likelihood vector.

Suppose  $\{s_1, \ldots , s_n\} \subset (0, \bar s]$ such that $0 < s_1 < \cdots  < s_n$ and
 $\bar s = \ln T/\tau$.  Since the inter-jump periods are independent (exponential) random variables,
the likelihood of $s_1, \ldots , s_n$ being the jump times during $[0,\bar s]$
under hypothesis $H_v$, is the product of the
likelihoods of the observed inter-jump periods, with an additional factor
of the likelihood of not seeing a jump in the last interval:
$$
\left( \prod_{i = 0}^{n-1} \theta_v^* (m+i)e^{-\theta_v^*(m+i)(s_{i+1} - s_i)}\right)e^{-\theta_v^*(n+m)(\bar{s} - s_n)}
$$
Thus, the log likelihood for observing this is (letting $s_0=0$):
\begin{align}  
 \Lambda^Z = n \ln \theta_v^*   -  \theta_v^*
  \left(m \bar s + \sum_{i = 1}^{n} (\bar s-s_i)\right)    \label{eq:loglikelihood}
\end{align}
(With $s_i = \ln(t_i/\tau)$,  \eqref{eq:loglikelihood} is the same as \eqref{eq:LambdaC_approx}, although in \eqref{eq:LambdaC_approx} the variables $t_i$ are supposed
to be integer valued.)
Let  $A_{\bar s}\triangleq \left(m \bar s + \sum_{i = 1}^{n} (\bar s-s_i)\right).$  Note that
$A_{\bar s}$ is the area under the trajectory of  $Z_{[0,\bar s]}$.  
Moreover, $n + m$ is the value of $Z_{\bar s}$.
So the log-likelihood vector is given by:
\begin{equation}\label{LLR}
    \Lambda^Z = (Z_{\bar s}-m)  \ln \theta_v^*    - (A_{\bar s}) \theta_v^*,
\end{equation}
which is a linear combination of $Z_{\bar s}-m$ and $A_{\bar s}.$   Thus,
the MAP decision rule has a simple form.   Let $f_Z^{C}(\rho, \theta^*, m, \bar s)$
denote the average error probability $p_e$ for the $MAP$ decision rule based on
observation of $Z_{[0,\bar s]}.$%\footnote{A linear scaling of time is equivalent to a linear scaling
%in the rate parameter for $Z$,  so $f_Z^{C}(\rho, \theta_1^*, \theta_0^*, m, \bar s) =
%f_Z^{C}(\rho, 1, \theta_0^*/\theta_1^*, m, \bar s/\theta_1^*$).}
%  Commented out footnote -- it's not important and a little distracting, 

There is apparently no closed form expression for the distribution of $\Lambda^Z,$
so computation of  $f_Z^{C}(\rho, \theta^*, m, \bar s)$  apparently requires
Monte Carlo simulation or some other numerical method.   A closed form
expression for the moment generating function of $\Lambda^Z$
is given in the following proposition, proved in Appendix \ref{app:proof_joint_ZA},
and it can be used to either bound the probability
of error or to accelerate its estimation by importance sampling.  

\begin{proposition}  \label{prop:joint_ZA_transform}
The joint moment generating function of $Z_{s}$ and $A_{s}$ is given as follows, 
where $\expectLm{\cdot}$ denotes expectation assuming the parameters of $Z$ are $\lambda,m:$
\begin{align}
\psi_{\lambda,m}(u, v, s) & \triangleq \expectLm{e^{uZ_{s} + vA_{s}}} \nonumber \\ 
& = \left( \frac{e^{(v - \lambda)s+u}} 
{1 + \frac{\lambda e^{u}}{v - \lambda}\left( 1 - e^{(v - \lambda)s} \right)} \right)^m .  \label{eq:joint_ZA}
\end{align}
\end{proposition}
Proposition \ref{prop:joint_ZA_transform} can be used to bound $p_e$ for
the special case of two possible labels, $r=2,$  in which estimating $\ell_{\tau}$
is a binary hypothesis testing problem: $H_1: \vartheta = \theta_1^*$, vs.
$H_2: \vartheta = \theta_2^*.$   For such a problem the likelihood vector
$\Lambda^Z$ can be replaced by the log likelihood ratio,
$\Lambda = \Lambda^Z(1) - \Lambda^Z(2).$
By a standard result in the theory of binary hypothesis testing
 (due to \cite{KobayashiThomas67}, stated without proof in \cite{Poor94Book},
 proved in special case $\pi_1=\pi_2=0.5$ in \cite{Kailath67}, and same proof
 easily extends to general case),
 %}{,}
 the probability of error for the MAP decision rule is bounded by
\begin{equation}  \label{eq:Bhatt}
\pi_1\pi_2   \rho_B^2 \leq p_e \leq    \sqrt{\pi_1\pi_2}  \rho_{B},
\end{equation}
where the Bhattacharyya coefficient (or Hellinger integral) $\rho_{B}$ is defined by
$\rho_{B} =  \expect{\eexp^{\Lambda/2}\big| H_2 },$  and $\pi_1$ and $\pi_2$ are
the prior probabilities on the hypotheses.
The proposition with $\lambda=\theta_2^*$, $u=\frac 1 2 \ln(\theta_1^*/\theta_2^*),$  $v=  - \frac{\theta_1^*- \theta_2^*} 2,$ and $s=\bar s$ yields
\begin{align*}
    \rho_{B,C} & = \expectLm{e^{u(Z_{s}-m) + vA_{s}}}   = \psi_{\lambda,m}(u, v, s) e^{-mu} \nonumber \\
 %    &= \left(   \frac{e^{(v - \lambda)s}} {1 + \frac{\lambda e^{u}}{v - \lambda}\left( 1 - e^{(v - \lambda)s} \right)}  \right)^m \\
     & =  \left( \frac{    \eexp^{-(\theta_1^* + \theta_2^*)\bar s/2} }
   { 1 - \frac{2 \sqrt{\theta_1^* \theta_2^*}}{\theta_1^*+\theta_2^*}\left(1 -   \eexp^{-(\theta_1^*+\theta_2^*)\bar s/2}\right) }
   \right)^m  .
\end{align*}
Here we wrote $\rho_{B,C}$ to denote it as the Bhattacharyya coefficient for Algorithm C  (for the large $T$ limit).
Using this expression in \eqref{eq:Bhatt} provides upper and lower bounds on
$p_e=f_Z^{C}(\rho, \theta^*, m, \bar s)$ in case $r=2.$

For the sake of comparison,  we note that the Bhattacharyya coefficient for the hypothesis testing
problem based on $\check Y_T$ alone,  i.e. Algorithm DT, is easily found to be:
\begin{align*}
\rho_{B,DT} = \left( \frac{    \eexp^{-(\theta_1^* + \theta_2^*)\bar s/2} }
   { 1 - \sqrt{  (1 -   \eexp^{-\theta_1^*\bar s})(1 -   \eexp^{-\theta_2^*\bar s}) } }
   \right)^m .  
\end{align*}

\section{Performance scaling for Algorithms C and DT}   \label{sec:perf_scaling}

Consider the community recovery problem for $m, r, \rho,$ and $\beta$ fixed, and large $T,$
such that the rate parameters $\theta_v^* : v \in [r]$ are distinct.
Let $\delta$ be an arbitrarily small positive constant.
The problem of recovering $\ell_\tau$ for some vertex $\tau$ with $\delta T \leq \tau \leq T$
 from $G_T$ using children (C)  (respectively, degree thresholding (DT)) is asymptotically
equivalent to the $r$-ary hypothesis testing problem for observation
$Z_{[0,\ln(T/\tau)]}$  (respectively, $Z_{\ln(T/\tau)}$)
with the same parameters $m$ and $\theta_v^* : v \in [r].$
This leads to the following proposition, based on the results on coupling of $Y$, $\tilde Y$ and $\check Y$
and the connection of $\check Y$ to $Z.$

%%%%%%%%%
\begin{proposition}   \label{prop:error_scaling} (Performance scaling for Algorithms C and DT)
(a)  Let $p_{e,\tau, T}^{(C)}$ denote the probability of error for recovery of the label $\ell_{\tau}$
using Algorithm C.   For any $\delta  \in (0,1)$, as $T\to \infty,$
\begin{align*}
\max_{\tau : \delta T\leq \tau \leq T}  | p_{e,\tau, T}^{(C)} -  f_Z^{C}(\rho, \theta^*, m, \ln(T/\tau)) |  \to 0.
\end{align*}
(b) Let $\hat p_{e,T}^{(C)}$ denote the fraction of errors for recovery of the labels of $G_T$ using
Algorithm C for each vertex.  Then,
\begin{align*}
\hat p_{e,T}^{(C)} & \overset{T\to\infty}\longrightarrow 
 \int_0^1   f_Z^{C}(\rho, \theta^*, m,  \ln(1/\delta)) d \delta,
 \end{align*}
 where the convergence is in probability.   \\
 (c) Parts (a) and (b) hold with C replaced by DT.
 \end{proposition}
 
 %   \section{Proof of Proposition \ref{prop:error_scaling}}   \label{sec:error_scaling}   % Used if moved to appendix.
 %  This proposition has a short proof and ties together much of the paper, so it'd be nice to keep proof in main body.
\begin{proof}
Observing the children of vertex $\tau$ in $G_T$ is equivalent
to observing $Y_{[\tau, T]}.$   In view of Proposition \ref{prop:YJ_couple},
the binary hypothesis testing problem based on observation of $Y_{[\tau, T]}$ is
asymptotically equivalent to the binary hypothesis testing problem based
on observation of $\tilde Y_{[\tau, T]}$ or on  $\check Y_{[\tau, T]}.$   The
upper bound on total variation distance is uniform for $T/\tau$ bounded.  In particular,
the minimum average probabilities of error for the problems become arbitrarily
close as $T\to \infty.$   To complete the proof of  (a), we next compare the probability
of recovery error based on observation of $\check Y$ vs. observation based on the
continuous time process $Z.$

The process $\check Y_{[\tau, T]}$ is obtained by sampling the process
$Z_{\ln(t/\tau) }$ at integer times $t \in [\tau, T].$   The mapping
from $Z$ to $\check Y$ does not depend on the parameter $\vartheta,$   which
could equal $\theta_v^*$ for any $v\in [r].$   In other words, observing $\check Y_{[\tau, T]}$ 
is equivalent to observing $Z_s$ for all $s \in [0, \ln(T/\tau)]$ such that $\tau \eexp^s$ is
an integer, where $Z$ has rate parameter $\theta_v^*$ under the hypothesis $\ell_\tau = v.$
Thus, in the terminology of source coding, 
$\check Y_{[\tau, T]}$ is a quantized version of $Z_{[0,\ln(T/\tau)]},$  with the
quantizer becoming arbitrarily fine as $\tau \to \infty.$   Therefore,  the
minimum probability of error for recovering $\ell_\tau$
based on the children of $\tau$ in $G_T$, in the limit as $\tau, T \to \infty$ with
$1 \leq T/\tau \leq 1/\delta$  is uniformly arbitrarily close to
$f_Z^{(C)}(\rho, \theta^*, m, \ln(T/\tau)) .$
This completes the proof of part (a).
Therefore, by the bounded convergence theorem and the fact $\delta$ can
be taken arbitrarily small,  convergence of the expected fraction of label errors follows:
\begin{align*}
\expect{ \hat p_{e,T}^{(C)}  }& \overset{T\to\infty}\longrightarrow 
\int_0^T  f_Z^{C}(\rho, \theta^*, m,  \ln(T/\tau)) d\tau  \\  & =
 \int_0^1   f_Z^{C}(\rho, \theta^*, m,  \ln(1/\delta)) d\delta.
 \end{align*}
 
 The last part of the proof is to show that the convergence is true not only in mean, but
 also in probability.   That follows by the same method used for
 the alternative proof of Proposition \ref{prop:empirical_m1}, about the empirical degree
 distribution, given in Appendix \ref{app:empirical_m1}.  The key step is a proof that
 the joint degree evolution processes $(\tilde Y^j)$ for a finite number $J$ of vertices (we
 only need to consider $J=2$ here) are asymptotically independent in the sense
 that the total variation distance to a process with independent degree evolution
 converges to zero.   That implies the error events for different labels
 are asymptotically uncorrelated, so convergence in probability to the mean
 follows by the Chebychev inequality.   The same proof works for $C$ replaced by $DT.$
 \end{proof}
 
We conjecture that a result similar to Proposition  \ref{prop:error_scaling} exists for  label recovery
using the message passing (MP)  algorithm described in the next section.
 
The following proposition, proved in Appendix \ref{app:small_tau}, addresses the case
that $\tau = o(T)$, including  the possibility that $\tau$ is a constant.     The estimation
procedure is a modification of Algorithm C.
\begin{proposition}   \label{prop:small_tau}  Suppose $T\to \infty,$ with
$\tau^o   \geq 1$ being a function of  $T$ such that $\tau^o  / T \to 0.$
Then $\ell_{\tau^o}$ can be recovered from knowledge of the
children of $\tau^o$  in $G_T$  with probability converging to one.
\end{proposition}

%%%%%%%%%  Numerical comparison  %%%%%%%%%%%%%%

\begin{example}[Numerical comparison for a single community plus outliers]  \label{examp:single_comm_examp}
Numerical results are shown in Figure \ref{fig:Combined}  for $m=5, r=2$, $\rho=(0.5,0.5)$
and  $\beta=\left( \begin{array}{cc} b & 1 \\  1 & 1 \end{array} \right),$ with $b=4$,  corresponding
to a graph with a single community of vertices and outlier vertices.
For these parameters, $ \eta^*= (0.622839 , 0.377121)$ and $\theta^*=(0.598612, 0.337153).$
There is little difference between the error probabilities of Algorithms DT and C
for $t/T \geq 10^{-1}$ but the difference is quite large for $t/T \leq 10^{-2}.$
Thus, for the vertices arriving in the top one percent of time, Algorithm C, which uses
the identity of children of a vertex, substantially outperforms Algorithm DT, which uses only the
number of children.  The  Bhattacharyya upper bounds are not very tight but the ratio of upper
bounds  for DT and C is similar to the ratio $f_Z^{DT}/f_Z^{C}.$ 
The derivative of $f_Z^{DT}(\rho, \theta^*, m, \ln(T/t))$ with respect  to $t/T$ has jump discontinuities at values
 of $t/T$ such that the threshold in the MAP test changes from one integer to the next, which is
 noticeable in the plot for $t/T$ close to 1, where the thresholds are small.
 \begin{figure}[htb]
\post{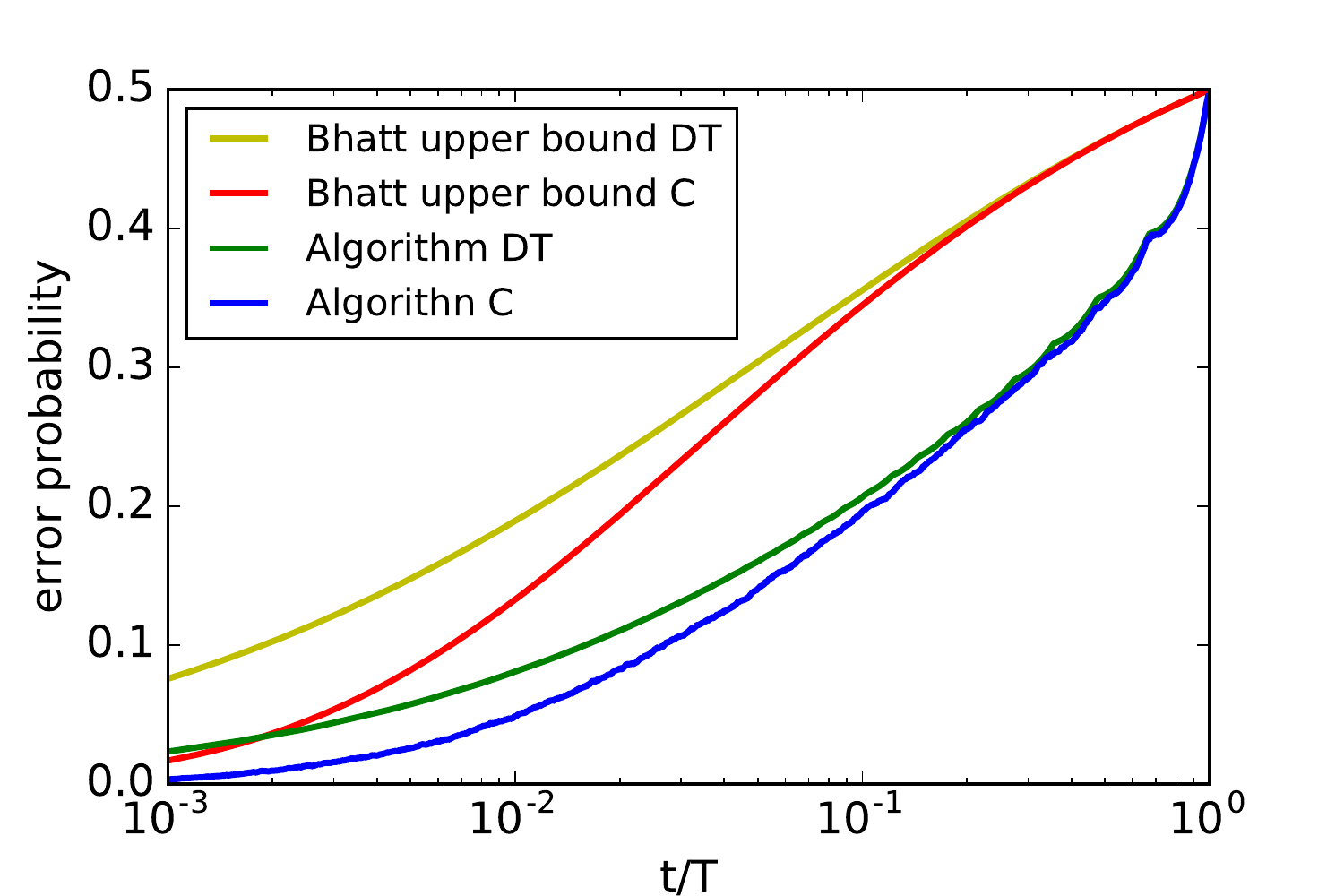}{8.5}
\caption{Semilog plot of Bhattacharyya upper bounds $\frac 1 2 \rho_{B,DT}$ and $\frac 1 2 \rho_{B,C},$
and functions $f_Z^{DT}$ and $f_Z^C,$ for an example with a single community
of vertices and outlier vertices.}
\label{fig:Combined}
\end{figure}
\end{example}

\section{Joint estimation of labels of a fixed set of vertices}   
\label{sec:joint_estimation}

The idea of algorithm $C$ is to estimate the label of a single vertex based
on the likelihood of  the observed set of children of the vertex, given the possible labels
of the vertex.   A natural extension, described in this section,
is to jointly estimate the labels of a small fixed set of vertices from the joint
likelihood of the children sets of the fixed set of vertices.
Given a vector of possible labels of the vertices in the set,  under the
approximation $\eta_t \equiv \eta^*$ for all $t,$  it is possible to compute
the joint likelihood of the children sets for
the vertices.    Maximizing over all label vectors gives an approximate
maximum likelihood estimate of the label vector.
We use the following notation.
\begin{itemize}
\item $V\subset \naturals,$  a finite set of vertices to be jointly classified
\item $b \in [r]^V$,  an assignment of labels for the vertices in $V$
\item $Y_t^\tau$ is the degree of vertex $\tau$ in $G_t.$
\item $A_t^\tau$ is the number of edges from vertex $t$ to vertex $\tau$
\item $A_t^{V^c}=m - \sum_{\tau \in V}   A_t^\tau$
\item Attachment of vertices in $[\bar t +1 , T]$ is observed, for some $\bar t$ and $T$
with $\max\{\tau : \tau \in V\} \leq  \bar t < T.$
\end{itemize}

\paragraph*{\bf Joint estimation algorithm}  The joint estimation algorithm
for estimating $(\ell_t : t \in V)$ is to calculate
\begin{align*}
\hat b_{\ML} = \arg\max_{b}  \ln P\left(\left(A_{[\bar t+1, T]}^{\tau}: \tau\in V \right)\bigg| b \right) ,
\end{align*}
using the the following approximate expression for the log likelihoods:
\begin{small}
\begin{align*}
&\ln P\left(\left(A_{[\bar t+1, T]}^{\tau}: \tau\in V \right) | b \right) \approx const  + \\
& \sum_{t=\bar t}^{T-1}   \ln \sum_{u\in [r]} \rho_u  
\left( \prod_{\tau\in V}  \left(   \frac{Y_t^{\tau} \theta^*_{u,b_{\tau}}}{mt}\right)^{A_{t+1}^{\tau}}  \right)
 \left( 1- \sum_{\tau'\in V}  \frac{Y_t^{\tau'} \theta^*_{u,b_{\tau'}}}{mt}\right)^{A_{t+1}^{V^c}},
\end{align*}
\end{small}

\noindent
where $const$ represents a constant not depending on $b$
(it is the sum of logarithms of multinomial coefficients) and the approximation
stems entirely from approximating $\eta_t$ by $\eta^*.$
We could calculate either the approximate ML estimator,
$\hat b_{\ML}$ by finding the arg max of the approximate log likelihood
with respect to $b$,  or $\hat b_{\MAP}$ in the same way but first adding the
log of the prior probability of $b.$   The complexity of the algorithm
is $\Theta(r^n nT),$  which is feasible for small values of $n.$

\begin{remark}    \label{rmk:joint_estimation}
By Proposition \ref{prop:YJ_couple}, if the set $V$ were to have a fixed
number of vertices, but the vertices depended on $T$ in such a way that
$V \subset [\delta T, T]$ for some fixed $\delta > 0$, then the sets of children
of the vertices would be asymptotically independent in the sense of total variation
distance.   Hence, in that limit, the joint estimation algorithm of this section would have
no better performance than Algorithm C.   That is why we envision using
the joint estimation algorithm for a fixed set of vertices as $T\to\infty.$

To see why joint estimation can help, consider two fixed vertices,
$\tau$ and $\tau'$  with $\ell_{\tau}=v$  and $\ell_{\tau'} = v'.$
By Proposition  \ref{prop:Y_large_time}
we expect the degrees of the two vertices at time $t$ to be on the order
of  $m(t/\tau)^{\theta^*_v}$ and  $m(t/\tau')^{\theta^*_{v'}}.$   Thus, if $m\geq 2$, the
probability of the two vertices having a common child at time $t$ to be
proportional to the product of their degrees divided  by $t^2$, or on the order
of  $(const) t^{ \theta^*_v + \theta^*_{v'} - 2}.$    Thus, if  $\theta^*_v + \theta^*_{v'} \geq 1$
we expect the number of common children of vertices $\tau$ and $\tau'$ in $G_T$ to converge to
infinity as $T\to\infty$,  with a constant multiplier that can thus be consistently
estimated as $T\to\infty.$   In particular, if $\theta^*_v = \theta^*_{v'} \geq 0.5$,  the
rate of growth of joint children would typically depend on whether the two vertices are
in the same community, providing consistent estimation whereas Algorithm C would fail.

\end{remark}

\section{The message passing algorithm}   \label{sec:message_passing}

In this section, we  describe how Alorithm C (the MAP rule given children) can be extended
to a message  passing algorithm.   We describe the algorithm for the case of $r\geq 2$ possible
labels for a general $r\times r$ matrix $\beta$ with positive entries, and  fixed $m \geq 1.$
{\em Throughout the remainder of this section,  let $(V, E)$ be a fixed instance of the random
graph, $(V_T, E_T),$ with known parameters $m, r, \beta, \rho,  t_o, G_{t_o},$ and $T.$}
 The message passing algorithm is run on this graph, with the aim of
calculating  $\Lambda_\tau$ for $1 \leq \tau \leq T,$ where for each $\tau$, 
$\Lambda_\tau$ is a log-likelihood vector:
\begin{align*}
\Lambda_\tau(v)  \triangleq \ln \prob{E_T = E \vert \ell_\tau = v} + const,   ~~~ v \in [r]
\end{align*}
where $const$ represents a constant that can depend on the graph but does not depend on
the vertex label $v.$    Then we can calculate the maximum likelihood (ML) and
 maximum a posteriori probability (MAP)
estimators of the label  of a vertex $\tau$ by
$\hat{\ell}_{\tau , ML}  =\arg\max_{v\in[r]}  \Lambda_\tau (v) $  and
$\hat{\ell}_{\tau , MAP}  =\arg\max_{v\in[r]}  \rho_v \Lambda_\tau (v).$

The messages in the message passing algorithm given below are also
log likelihood vectors, so  two values, $\nu, \nu' \in \reals^r$,  of such a message are considered
to be equivalent if $\nu - \nu'$ is proportional to the all ones vector in $\reals^r.$
For example, given a log likelihood vector $\nu$ there is a
{\em canonical equivalent log likelihood vector} $\nu'$
such that $\max_{u\in [r]}  \nu'(u) =0$, namely,  
$\nu'$ defined by $\nu'(u) = \nu(u) - \max_{u'\in[r]} \nu(u').$       This fact is useful for
numerical computation; in our computer code we stored all log likelihood vectors in
their equivalent canonical  forms.    A log likelihood vector is said to be a {\em null log likelihood vector}
if it is a constant multiple of the all one vector.   In other words, a null log likelihood vector
is equivalent to the zero vector.
 In the special case $r=2$, $\Lambda_{\tau}(1) - \Lambda_{\tau}(2)$
and  $\nu(1) -\nu(2)$ represent log likelihood ratios,  and the algorithm below can easily
be restated using real valued messages that have interpretations as log likelihood ratios
instead of using length two log likelihood vectors.

A complete specification of a message passing algorithm includes specification
of the following elements:
\begin{enumerate}
\item initial messages
\item mappings from messages received at a vertex to messages sent by the vertex
\item timing of message passing and termination criterion
\item mappings from messages received at a vertex to the output log likelihood vector of the vertex
\end{enumerate}
About element 3).   A natural choice for the timing of message passing is synchronous.
For synchronous timing, all messages to be sent along each edge of the graph $G_T$
(excluding edges in the initial graph $G_{t_o}$)  are computed.  Based on those,
log likelihood vectors are computed for each vertex and the next round of messages
to be sent is computed.  An alternative timing of messages is to alternate between
updating only messages from children to parents and updating only messages from parents
to children.   For termination, we stopped the  message passing when the sum of Euclidean norms
of differences in the canonical log likelihood vectors was below a threshold.

In this section we specify the equations for elements 1), 2), and 4).

 Given vertices $\tau$ and $\tau_0$, we say
$\tau$ is a child of $\tau_0$, and $\tau_0$ is a parent of $\tau$, if $\tau \geq \max\{\tau_0, t_o\}+1,$
and there is an edge from $\tau$ to $\tau_0.$    It is assumed that the
known initial graph $G_{t_o}$ is arbitrary and carries no information
about vertex labels.   Thus,  for the inference problem at hand,  the edges
in $G_{t_o}$ are not relevant beyond the degrees that they
imply for the vertices in $G_{t_o}.$
 Let $\partial \tau$ denote the children of $\tau$ in $G_T$ and $\wp \tau$
 the parents of $\tau.$    So  $\wp \tau = \emptyset$ if $\tau \leq t_o$ and
 $\partial \tau \subset \{t_o+1,  \ldots, T\}.$
 Let $\nu_{\tau \rightarrow \tau_0}$ denote a message passed from
 child to parent, and $\mu_{\tau_0 \rightarrow \tau}$ denote a message passed from
 parent to child.

Let $g^{cp} : \reals^r \mapsto \reals^r$ and  $g^{pc} : \reals^r \mapsto \reals^r$
be defined as follows (here ``cp" denotes child to parent, and ``pc" denotes parent to child)
\begin{align*}
g^{cp}(\nu)(v) &=    \ln   \left(  \sum_{u\in [r]}    \eexp^{\nu(u)}   \rho_u \theta_{u, v}^*/\theta^*_v  \right)     
         ~~\mbox{for } \nu \in \reals^r      \\
%g^{cp}(\nu)(v) =    \ln   \left(  \sum_{u\in [r]}    \eexp^{\nu(u)}   \rho_u \theta_{u, v}^*  \right)     
%        ~~\mbox{for } \nu \in \reals^r      \\
g^{pc}(\mu)(v)  &=   \ln \left( \sum_{v'\in[r]}  \theta_{v,v'}^*    \eexp^{\mu(v')} \rho_{v'}/\theta^*_{v'}  \right)
          ~~\mbox{for } \mu \in \reals^r ,
\end{align*}
where $\theta^*_{u,v}$ and  $\theta^*_u$ are defined in Section \ref{sec:emp_degree_dist}.
For convenience, we  repeat the expression in \eqref{eq:LambdaC_approx} for the approximate
log likelood vector based on observation of children:       
\begin{align}
&\lambda^C_\tau (v) = |\partial \tau|\ln \theta^*_v
+ \theta_v^* \left( d_0(\tau)    \ln\frac{\tau\vee t_o}{T} 
   + \sum_{t \in \partial \tau} \ln \frac{t}{T}\right),  \label{eq:lambda_eq} 
%&\lambda^C_\tau (v) =  \theta_v^* \left({\sf deg}_{{G}_{\tau\vee t_o}}(\tau)    \ln\frac{\tau\vee t_o}{T} 
%   + \sum_{t \in \partial \tau} \ln \frac{t}{T}\right),  \label{eq:lambda_eq} 
 \end{align}
 where $\tau \vee t_o = \max\{ \tau, t_o\}$ and
 $d_0(\tau)$ is the initial degree of vertex $\tau$,
 defined to be the degree of $\tau$ in $G_{t_o}$
 if $\tau \leq t_o$ and $d_0(\tau)=m$ otherwise.
The message passing equations are given as follows.
See Appendix \ref{sec:derivation_of_MP} for a derivation.
\clearpage
\begin{align}
%%%%%  child to parent messages
 &   \nu_{\tau \rightarrow \tau_0}=  \lambda^C_\tau  + \sum_{t\in \partial \tau}
  \tilde  \nu_{t\to\tau}
+  \sum_{\tau_1 \in \wp \tau \backslash \{\tau_0\} }
\tilde \mu_{\tau_1 \to \tau}
          \label{eq:child_to_parent_a} \\ \nonumber  \\
%           &   \nu_{\tau \rightarrow \tau_0}=  \lambda^C_\tau  + \sum_{t\in \partial \tau}
%  \tilde  \nu_{t\to\tau}
%+  \sum_{\tau_1 \in \wp \tau \backslash \{\tau_0\} }
%\tilde \mu_{\tau_1 \to \tau}
%         \label{eq:child_to_parent_a} \\ \nonumber  \\
 %%%%%%  parent to child messages
&  \mu_{\tau_0 \rightarrow \tau}
=  \lambda^C_{\tau_0} + \sum_{t\in \partial \tau_0\backslash \{\tau\}}   \tilde \nu_{t\to\tau_0}
 +  \sum_{\tau_1 \in \wp \tau_0} 
 \tilde \mu_{\tau_1 \to \tau_0} 
   \label{eq:parent_to_child_a} \\ \nonumber \\
%%%%%%%  inversion of messages
&  \tilde \nu_{\tau \rightarrow \tau_0} = g^{cp} (  \nu_{\tau \rightarrow \tau_0}  )  
     \label{eq:child_to_parent_b}    \\
&  \tilde \mu_{\tau_0  \to \tau}  = g^{pc} ( \mu_{\tau_0  \to \tau}  ) 
      \label{eq:parent_to_child_b}    \\  \nonumber \\
 %%%%%%%%  lambda equation
&   \Lambda_\tau =  \lambda^C_\tau + \sum_{t\in \partial \tau} 
       \tilde \nu_{t\to\tau}   +  
   \sum_{\tau_0 \in \wp \tau}   \tilde \mu_{\tau_0 \to \tau} ,
  \label{eq:combining}  
\end{align}
with the initial conditions:
\begin{align}  \label{eq:initialize_tilde}
 \tilde \nu_{\tau \to \tau_0} =  0  ~~~~
 \tilde \mu_{\tau_0 \to \tau} = 0,
\end{align}
or equivalently
\begin{align}  \label{eq:initialize}
~~~\nu_{\tau \to \tau_0} =   \lambda^C_{\tau} ~~
\mu_{\tau_0 \to \tau} =  \lambda^C_{\tau_0} .
\end{align}

In \eqref{eq:child_to_parent_a} - \eqref{eq:combining}  messages with the
letter $\nu$ are sent from child to parent, and messages with letter $\mu$ are
sent from parent to child.   The $r$ coordinates of a message without  a tilde
represent likelihoods  given possible labels of the
sending vertex, while the $r$ coordinates of a message with a tilde  represent
likelihoods given possible labels of the receiving vertex.
The equations could be written entirely using only the $\nu$'s
and $\mu$'s by applying  \eqref{eq:child_to_parent_b} and
\eqref{eq:parent_to_child_b}  within   \eqref{eq:child_to_parent_a} and
\eqref{eq:parent_to_child_a}.    Or the equations could be written entirely
using only the 
$\tilde \nu$'s
and $\tilde \mu$'s by applying  \eqref{eq:child_to_parent_a} and
\eqref{eq:parent_to_child_a}  within   \eqref{eq:child_to_parent_b} and
\eqref{eq:parent_to_child_b}.  

The edges in the initial graph $G_{t_o}$ are not relevant in the algorithm beyond
the fact they determine the degrees of the vertices in $G_{t_o}.$
The message passing equations are written as if there are no parallel edges in $(V,E).$
While the fraction of edges that are parallel to other edges will be small for large $T$, they are
permitted.  The convention used in the message passing algorithm is that $\partial \tau$
and $\wp \tau$ are to be
considered as multisets, so that if a vertex appears with some multiplicity in one of those
sets, then the corresponding term in the summations will be appearing the corresponding
number of times.  

\begin{remark}
The  {\em fitness only} case of the preferential attachment model with
communities occurs if either of the following two equivalent conditions hold:
\begin{enumerate}
\item  $\beta$ has rank one
\item  $\theta^*_{u,v} = \theta^*_v$ for all $u.$ 
\end{enumerate}
Since the distribution of the preferential attachment model
with communities does not change if a row of $\beta$ is multiplied
by a positive constant, for the fitness only case of the model it
could be assumed that the rows of $\beta$ are identical.

In the fitness only case of the model, 
both $g^{pc}$  and $g^{cp} $ map to null log likelihood vectors for
any choice of their arguments,  so all messages generated in the
message passing algorithm are null log likelihood vectors.
Consequently, if $\beta$ has rank one then the
message passing algorithm converges in one iteration
and it coincides with algorithm $C.$
\end{remark}

\section{Monte Carlo simulation results}
\label{sec:simulations}

The simulation results reported in this paper were
computed for random graphs with $m=5$, 
$\rho_u = 1/r$ for $u\in [r]$, and two vertices in the
initial graph (i.e. $t_o=2$) with degree $2m$  each.   The
specific choice of initial edges is not relevant, but
there could for example be $2m$ parallel edges between
the two initial  vertices, or for example each of the two vertices
could have $m$ self loops.

\subsection{Single community} 

 The performance of the message passing algorithm is described for the case
of a single community plus outliers, described in Example \ref{examp:single_comm_examp}.
Through numerical experimentation, we found the following timing of message passing
works well.  We take the initial values of all $\tilde \mu$ and $\tilde \nu$ messages to be zero.
For the timing of message passing we run two phases.  In the first phase the messages from children
to parents (i.e. the $\tilde \nu$'s) are  repeatedly updated, while messages from parents to children
are held fixed.   In the second phase the messages $\tilde \nu$ are held fixed and the
messages from parents to children are repeatedly updated until the messages
converge.   In both phases the messages converge in a finite number of iterations.   After both
phases are completed, the (approximate) likelihood ratios are computed.   Numerical results
are shown in Figure \ref{fig:m5beta4_compare}.   The message passing algorithm significantly
outperforms the other two algorithms.   Another version of algorithm with about the same
performance is to use synchronous scheduling of all messages, while applying the message
balancing method described in Section \ref{sec:symmetric_mult_commun}.
\begin{figure}[htb]
\post{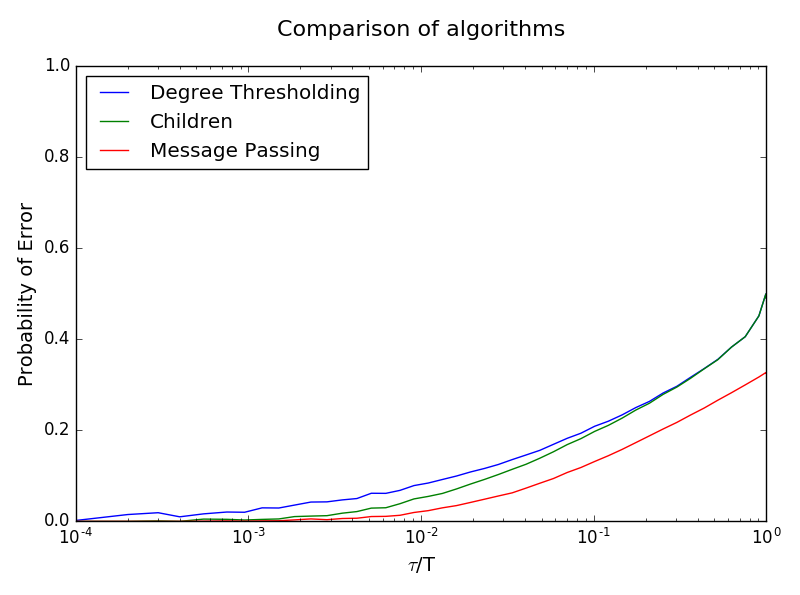}{8.5}
\caption{Semilog plot of error probability vs. vertex index for algorithms DT,C, and MP for
single community example with $m=5$, $\rho=(0.5,0.5)$, and  $b=4.$  The average
over 1000 runs of MP is shown.}
\label{fig:m5beta4_compare}
\end{figure}

\subsection{Symmetric multiple community graphs} \label{sec:symmetric_mult_commun}

To model the situation that each vertex is in one of $r$ communities with equal probability, with equal
affinities within each community, let $\rho_v = 1/r$ for $v\in [r]$ and, for some $b > 1,$
\begin{align*}
\beta_{u,v}
 =\left\{ \begin{array}{cl}
b & \mbox{if } u=v  \\
 1 & \mbox{else}
\end{array} \right. .
\end{align*}
Then  $\eta^*=\rho$,  $\theta_{u,u}^*=\frac{br}{2(b +r-1)}$ and, for $u \neq v,$
 $\theta_{u,v}^*=\frac r {2(b +r-1)}.$    Also,  $\theta_v^* = 0.5$ for all $v.$
 Note that $\lambda^C_\tau$ is a null log likelihood vector for all $\tau.$
Up to equivalence of log likelihood vectors (i.e. ignoring addition of constant multiples
of the all one vector)  $g^{cp}(\cdot) = g^{pc}(\cdot) = g(\cdot),$
where
\begin{align*}
g(\nu)(v)
 = \ln \left( b  \eexp^{\nu(v)} + \sum_{v' \in [r]\backslash\{v\}} \eexp^{\nu(v')} \right).
\end{align*}
In the special case $r=2$, the messages can be taken to be scalars representing
log likelihood ratios,  with $g$ taking the form
$g(\mu) \triangleq  \ln \frac {b  \eexp^\mu +1 }{\eexp^\mu + b}$.

The functions $g^{cp}$ and $g^{pc}$ map
null log likelihood vectors to null log likelihood vectors, so all messages equal to
null log likelihood vectors is a fixed point of the message passing
equations \eqref{eq:child_to_parent_a} - \eqref{eq:parent_to_child_b}.
Community detection is apparently rather difficult for this model in case $m=1$ because
$G_T$ is a tree and for the symmetric two or more community graphs the local neighborhood of
a vertex does not indicate which community the vertex is in, at least under the idealized
assumption $\eta_t \equiv \eta.$    We restrict attention to the case $m\geq 2.$  In that
case,  we can apply the joint estimation algorithm given in Section \ref{sec:joint_estimation}
to identify the labels of a small number of vertices, which we call {\em seeds}
to help initialize the message passing algorithm.       Accordingly, for the message passing
algorithm, we assume that the labels of the seed vertices are correctly  revealed to the algorithm.   
 Accordingly, the $\mu$ and $\nu$ messages sent by a seed vertex
 $\tau$ with $\ell_{\tau} = u$ would all be the same, and be given by:
 \begin{align*}
\nu_{\tau \to \tau_0}(v) =\left\{ \begin{array}{cl}
0 & \mbox{if } v=u    \\
-\infty & \mbox{else}
\end{array} \right.
\end{align*}
All other messages are initially set to zero. At every iteration,
 all the messages (both $\mu$ and $\nu$) are updated synchronously. 

One other technique, we call {\em message balancing}, was employed to get the
algorithm to give good performance.    Intuitively, the idea is to balance
the total amount of negativity about each community within the messages.
The following description of message balancing assumes the
messages are stored in their equivalent canonical form, described
near the beginning of Section \ref{sec:message_passing}.
At the beginning of each iteration, the $\tilde \mu$ messages
are scaled by a positive vector $f$:   $\tilde \mu_v \to f_v \tilde \mu_v.$
The scale vector $f$ is chosen for the iteration so that the sum of all the scaled $\tilde \mu$
messages is a null log likelihood vector (i.e. multiple of all ones vector) and the
sum of the messages is preserved.   The $\tilde \mu$ messages are similarly
scaled.   Empirically we found similar performance if only the $\tilde \mu$
messages were scaled, or if only the $\tilde \mu$ messages sent by seeds
were scaled.

We first present numerical results for an example with two communities
for  $T = 10,000 , m = 5,$ and $b = 4.$   We first describe the performance
of the joint estimation algorithm for estimating the labels of the first ten
vertices, taken to be seed vertices, and then describe the performance of the
message passing algorithm assuming the seed vertices are correctly classified.
The performance of the joint estimation algorithm is shown in
Figure \ref{fig:joint_est_2048samples_t20}.   Two different methods of
determining which ones of vertices 2 through 10 are in the same
community as vertex 1 were used.   The first method, called ``partial data" in the figure,
estimates the label of each vertex $\tau$ with $2\leq \tau \leq 10$
 by jointly estimating labels for the set of two vertices $V=\{1,\tau\},$
while the second method, called ``complete data" in the figure, is to jointly infer
the labels of vertices in $V=\{1, 2,  \ldots , 10\}.$   The value $\bar t=20$ was
used.   It was observed that the last term in the likelihood expression is sometimes
negative (a result of the approximation $\eta_t \equiv \eta$) for some values of
$t$ and $b.$    That was only observed in the simulations for some values of
$t$ with $t\leq 30.$   If for some $t$ a negative likelihood was observed for
some $b$, then the likelihood term for that $t$ was dropped for all vectors
$b.$   The performance  gives good evidence that for $n$ fixed, the labels of the
first $n$ vertices can be inferred with error probability converging to zero as $T\to\infty,$
for the symmetric two community model.
 \begin{figure}[htb]
\post{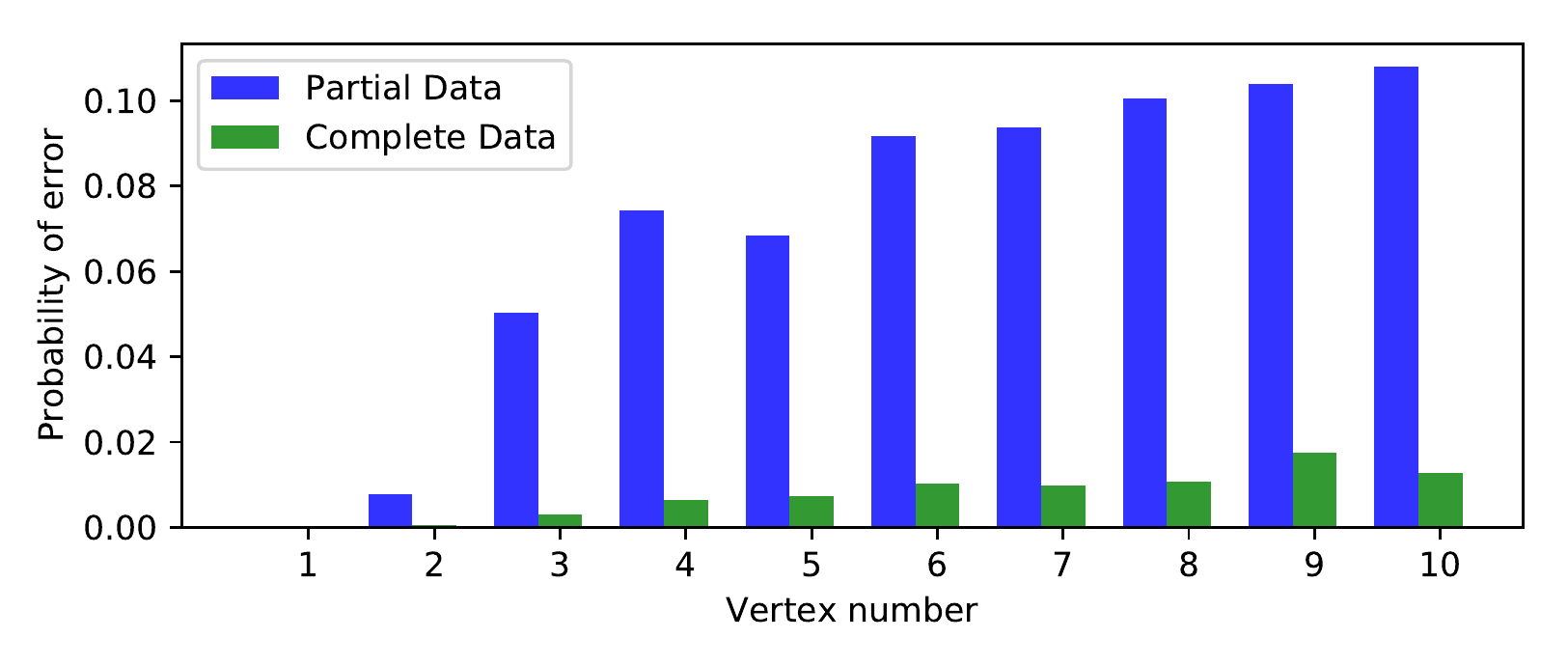}{9}
\caption{Error probabilities for determining whether vertices 1 and $\tau$ are in the
same community, for $2\leq \tau \leq 10,$   assuming symmetric two community
model with parameters $b=4$, $m=5$, $T=10,000.$
Error probabilities are shown for (a) estimation based on joint likelihoods given labels
for two vertices at a time (i.e. vertices  1 and $\tau$ with $2\leq \tau\leq 10),$ and
(b) for estimation based on approximate maximum likelihood
estimate of labels of vertices 1 through 10 simultaneously.   Error probabilities are estimated
by fraction of errors in 2048 simulations of graph, for estimation
based on children with time of arrival $t$ in the interval $[20, 10^4].$}
\label{fig:joint_est_2048samples_t20}
\end{figure}

Next Figure \ref{fig:sym_com_m5b4_10known} shows the performance of the
message passing algorithm run on 100 graphs
of size $T = 10,000$, with parameters $m = 5, b = 4$ with  two communities with
ten seed vertices.  The message passing algorithm is run until  the norm of the difference in
the vector of log-likelihoods is less than 1.
The probability of error curve plotted for each random graph is averaged over bins of width
increasing with  time. The ends of the bin intervals are chosen as a geometric progression
with factor 1.2.   Although there were only ten seed vertices, the algorithm nearly always correctly
classified the first 100 vertices, and also most of the first 1000 vertices.
\begin{figure}[htb]
%\post{sym_com_m5b4r2_10known}{8.5}
\post{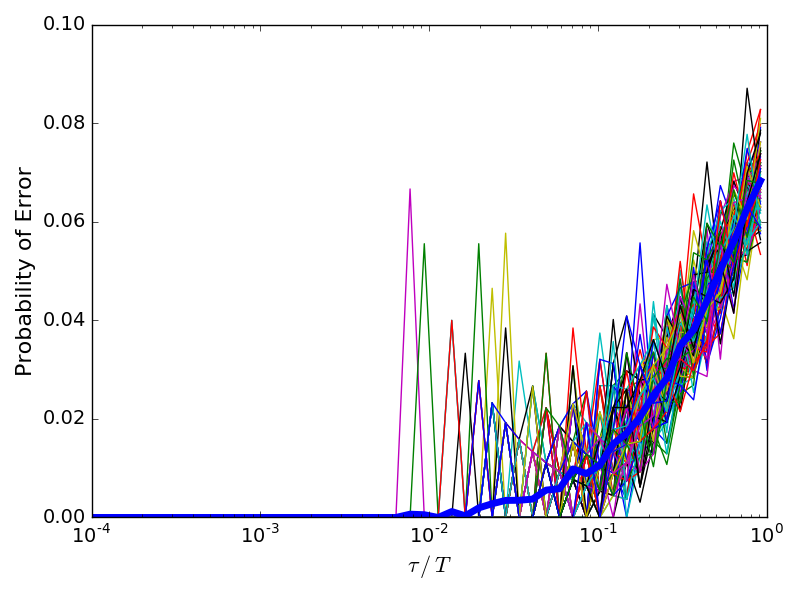}{8.5}
\caption{Semilog plot of error probability vs. vertex index for algorithm MP for
symmetric two ($r=2$) community graphs with   $m=5$ and  $b=4.$ 
The algorithm was given labels of the first ten vertices and message balancing was
used.   Smoothed results for 100 graphs are shown, with the average of them
represented by the thicker blue curve.
}
\label{fig:sym_com_m5b4_10known}
\end{figure}

Performance of the message passing algorithm for four communities with 20 seed vertices is shown in
Figure  \ref{fig:Gsym_fourcomm_b4_MP_20seeds}.   The result of running on 100 sample graphs is
shown.  The algorithm had poor performance for one sample labeled graph, for which one of the
communities was not represented  among the seeds.  In other simulations we have seen the algorithm
fail occasionally even if all communities are represented among the seeds.
\begin{figure}[htb]
\post{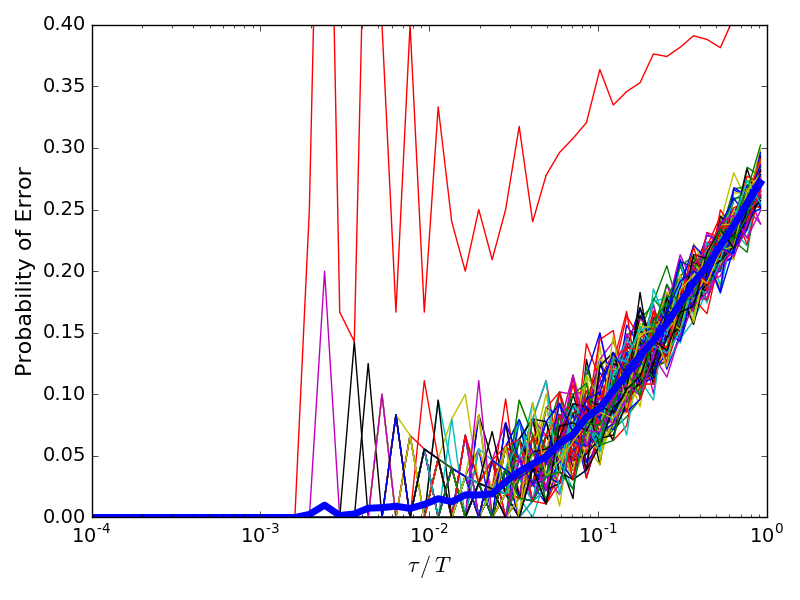}{8.5}
\caption{Semilog plot of error probability vs. vertex index for algorithm MP for
symmetric four ($r=4$) community graphs with   $m=5$ and  $b=4.$  The performance
for MP run on 100 independently generated graphs is shown.
The algorithm was given labels of the first twenty vertices and message balancing was
used. Smoothed results for 100 graphs are shown, with the average of them
represented by the thicker blue curve.}
\label{fig:Gsym_fourcomm_b4_MP_20seeds}
\end{figure} 

\subsection{Three communities with symmetry between two of them}

Consider three communities 1,2,3 such that each vertex is equally likely to be
in any of the three communities.     Vertices in community 1  have a growth
rate distinct from the growth rates of the other two communities, and the other
two communities are statistically identical.   We again begin with the joint estimation
algorithm, because identifying seed vertices can help the message passing algorithm
distinguish vertices in the two statistically identical communities.   To display the
performance of the joint estimation algorithm we need to adjust for the fact
that the assignment of labels 2 vs. 3 to the two symmetric communities is
arbitrary.   Thus, before computing errors, we see whether swapping the 2's and 3's
of the output label vector reduces the number of errors.   If yes, the 2's and 3's
of the output vector are swapped.   If there is a tie, with probability 0.5,
the 2's and 3's are all swapped.   Then, for each seed vertex, we say a {\em big
error} is made if the true label is 1 and the estimate is not 1, or vice versa.
We say a {\em small error} is made if both the true label and estimated label
are in $\{2,3\}$ but they are unequal.    The event that the label of a seed
vertex is in error is the disjoint union of a big error event and small error event.
The message passing algorithm was run using synchronous message timing
with 15 seed vertices and message balancing.

Two different $\beta$ matrices were tried, which we list with their corresponding
vectors $(\theta^*_v)$
\begin{align*} 
&\beta^I = \left(
\begin{array}{ccc}  
2 & 1 & 1 \\ 1 & 4 & 1 \\ 1 & 1 & 4
\end{array}
 \right)  ~~~~~~~~~~~~~~
\beta^{II} = \left( \begin{array}{ccc}  
4 & 1 & 1 \\ 2 & 4 & 1 \\ 2 & 1 & 4
\end{array} \right) \\
&(\theta^*)^{I}= (0.420,   0.532,   0.532]) ~~
(\theta^*)^{II}= (0.590, 0.438, 0.438) 
\end{align*}
For version I of the model,
Figure \ref{fig:G1L2H_joint_5seeds1000iter} displays the performance
of the joint estimation algorithm and Figure \ref{fig:G1L2H_MP_15seeds}
displays the performance of the message passing algorithm for 15 seed
vertices.   Proposition \ref{prop:Y_large_time} implies that as $T\to\infty$ the probability
of big errors converges to zero.     The probability of small errors is
apparently small for this model and algorithm.

\begin{figure}[htb]
\post{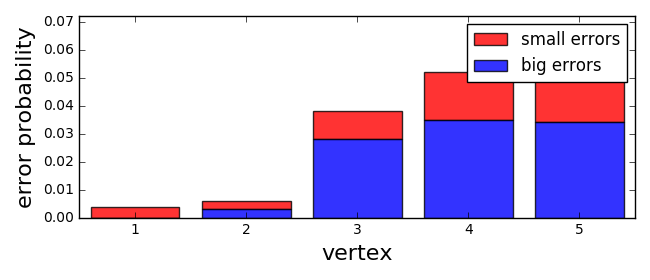}{8.5}
\caption{Big errors and small errors for joint estimation of the labels
of first five vertices for version I of the three communities example, estimated
using 1000 sample graphs.  At least one label is incorrect in 0.139 fraction of
graphs.}
\label{fig:G1L2H_joint_5seeds1000iter}
\end{figure} 

\begin{figure}[htb]
\post{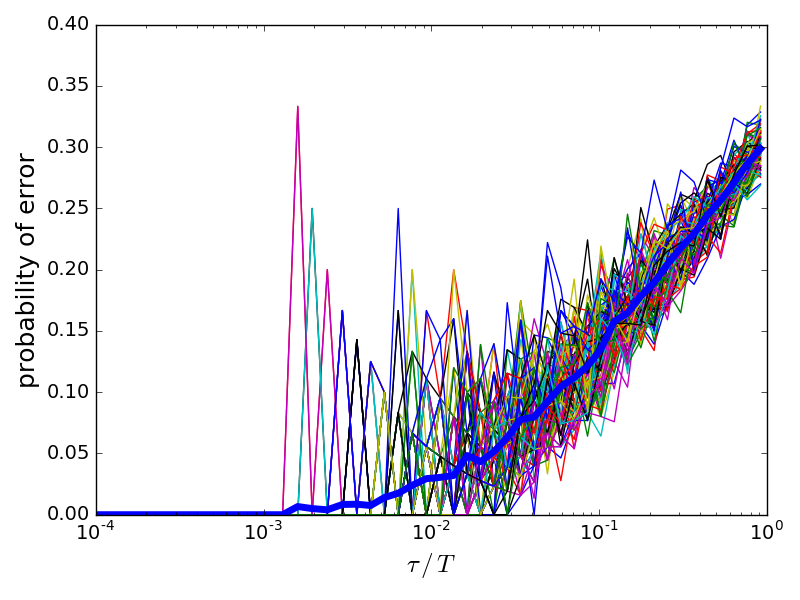}{8.5}
\caption{Error probabilities by vertex for version I of the
three communities example, for message passing with
15 seed vertices.  Smoothed results for 100 graphs are
shown, with the average of them represented by the thicker
blue curve.}
\label{fig:G1L2H_MP_15seeds}
\end{figure} 

For version II of the model,
Figure \ref{fig:G1H2L_joint_5seeds1000iter} displays the performance
of the joint estimation algorithm and Figure \ref{fig:G1H2L_MP_15seeds}
displays the performance of the message passing algorithm for 15 seed
vertices.    There are many more small errors for version II of the model than
for version I, which is explained by the fact that for version II, the two
equal sized communities that can't be distinguished by growth rates
alone (because $\theta_2^*=\theta_3^*$) have much smaller degrees
than vertices in the two equal sized communities of version I.
In fact, we conjecture that the probability of small errors does
not converge to zero for the joint estimation algorithm for version II.
The reason is that the mean number of common children of two
vertices that have labels in $\{2,3\}$ is stochastically bounded above
as $T\to\infty,$  because $\theta^*_v + \theta^*_{v'} < 1$ for $v, v' \in \{2,3\}.$
See Remark   \ref{rmk:joint_estimation}.

\begin{figure}[htb]
\post{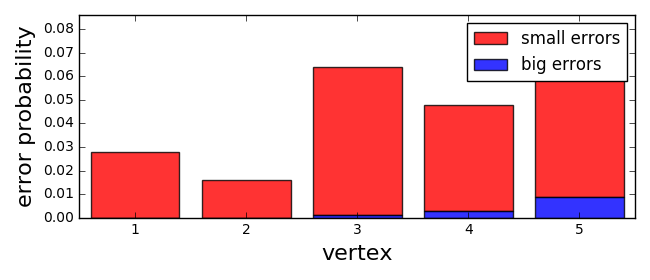}{8.5}
\caption{Big errors and small errors for joint estimation of the labels
of first five vertices for version II of the three communities example ,
estimated using 1000 sample graphs.}
\label{fig:G1H2L_joint_5seeds1000iter}
\end{figure} 

\begin{figure}[htb]
\post{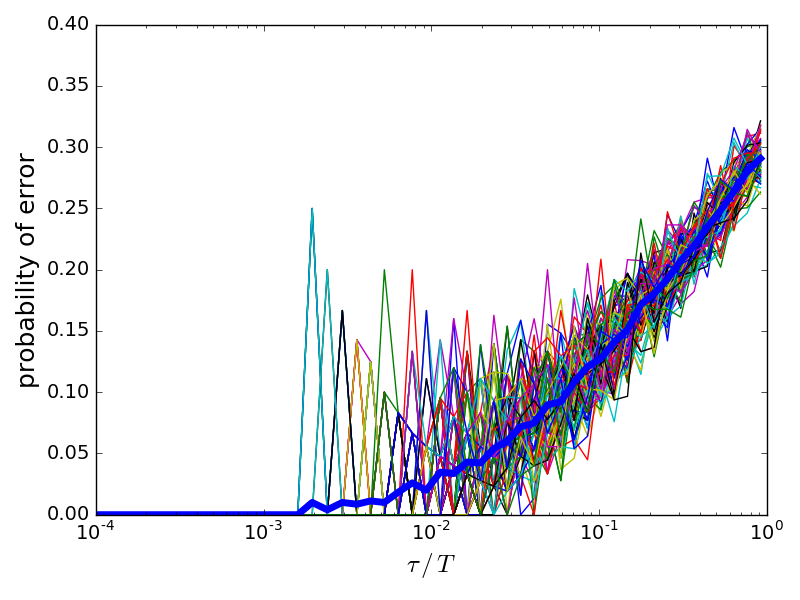}{8.5}
\caption{Error probabilities by vertex for version II of the
three communities example, for message passing with
15 seed vertices.}
\label{fig:G1H2L_MP_15seeds}
\end{figure} 

\section{conclusion}

The message passing algorithm, together with seeding by the joint inference
algorithm and balancing method, appear to work well in Monte Carlo simulations.
The use of seeds takes advantage of the large degrees of a few vertices.
The performance of the joint inference algorithm is related to the large time degree evolution
of one or more fixed vertices $\tau$ such that $T/\tau \to \infty$ as $T\to \infty,$
 whereas the derivation of the message passing algorithm is based on the joint degree
 evolution for one or more vertices $\tau$ such that $\tau \to \infty$ and $T/\tau$ remains
 bounded.   As version II of the three community example points out, it may not always
 be possible to consistently recover a fixed set of vertex labels as $T\to\infty$, while it
 is possible if the parameters $\theta^*_v: v\in [r]$ are distinct.
\clearpage
\appendices

%%%%%%%%%%%%%
\section{Proof of Proposition  \ref{prop:gobal_convergence}}   \label{sec:global_convergence}

Simple algebra yields
\begin{align}   \label{eq:eta_algebra}
\eta_{t+1} -  \eta_t  =  \frac{C_{t+1}-C_t -2m \eta_t}{2m(t+1)} .
\end{align}

The conditional distribution of  $C_{t+1}-C_t$ given $C_t$ and given $\ell_{t+1}=u$ can be represented using
a random variable with a multinomial distribution as
$$
C_{t+1}-C_t    \overset{d.}{=}
  me_u  + \mbox{multinom}\left( m ,   \left(
    \frac{\beta_{uv}\eta_{tv}}{\sum_{v'} \beta_{uv'}  \eta_{tv'} }
     : v\in [r]  \right)    \right)  ,
$$
where $e_u$ is the unit length $r$ vector with $u^{th}$ coordinate equal to one.
Therefore,
\begin{align}
E[C_{t+1,v}-C_{t,v}| C_t ] & =   m \rho_v  +  \sum_u   m\rho_u \left(  \frac{\beta_{uv}\eta_{tv}}{\sum_{v'} \beta_{uv'}  \eta_{tv'} }  \right)
\label{eq.Cdrift_gen}
 \end{align}
Combining with \eqref{eq:eta_algebra} yields that
\begin{align}
E[\eta_{t+1}-\eta_{t}| C_t ]  = \frac{1}{2(t+1)} h(\eta_t).
 \end{align}
 This gives the representation
\begin{align}
\eta_{t+1}  = \eta_t + \frac{1}{2(t+1)}\left[  h(\eta_t) + M_t   \right]   \label{eq:martingale_rep}
\end{align}
where
\begin{align}
M_t &  = C_{t+1}-C_t - E[C_{t+1}-C_t | C_t ]  .  \label{eq:martingale_def}
\end{align}
Note that $M$ is a bounded martingale difference sequence; $\prob{\norm{M_t}_1 \leq 4m}=1$ for all $t.$ 
Also, the Jacobian matrix of $h$ is uniformly bounded over the domain of probability vectors
so $h$ is Lipschitz continuous. 
In view of \eqref{eq:martingale_rep} and these properties,
 the theory of stochastic approximation implies the possible limit points of
$\eta_t$ is the set of stable equilibrium points of the
ode  $\dot \eta = h(\eta)$  \cite[Chapter 2, Theorem 2]{borkar2008stochastic} .

Since $\sum_v h_v(\eta) \equiv 0,$  the ode $\dot \eta = h(\eta)$ can be restricted to the space of probability vectors.
A Lyapunov function is used in \cite{Jordan13} to show that the
ode $\dot \eta = h(\eta)$ restricted to the space of probability vectors
 has a unique globally stable equilibrium point, which we denote by $\eta^*.$ 
 
 %%%%%%%%%%%%%%%%
 \section{Proof of Proposition \ref{prop:Y_couple} }    \label{app:proof_of_Y_couple}

\begin{remark}   \label{rmk:on_variation_distance}
 (i) We shall use extensively the connection between total variation distance and coupling.
Given two discrete probability distributions $a$ and $b$ on the same discrete set, the total
variation distance between $a$ and $b$ is defined by
$d_{TV}(a,b)=\frac 1 2 \sum_i  |a_i - b_i| .$    If  $A$ and $B$
are random variables, not necessarily on the same probability
space, we write $d_{TV}(A,B)$ to represent $d_{TV}(\calL(A), \calL(B))$,
which is the total variation distance between the probability distributions
of $A$ and $B.$
Clearly $d_{TV}$ is a distance metric; in particular
it satisfies the triangle inequality.
An operational meaning is
$d_{TV}(a,b)  = \min \prob{A\neq B},$ where the minimum is taken
over all pairs of jointly distributed random variables $(A,B)$
such that $A$ has distribution $a$ and $B$ has distribution $b.$
In other words, $d_{TV}(a,b)$ is the minimum failure probability when
one attempts to couple a random variable with distribution $a$ to a random
variable with distribution $b.$  \\
(ii)  The distance $d_{TV}(a,b)$ can be expressed as
\begin{align}   \label{eq:TV_alternate}
d_{TV}(a,b)= \sum_i  ( b_i - a_i )_+
\end{align}
Expression \eqref{eq:TV_alternate} is especially useful if $b_i > a_i$ for only a
small set of indices $i.$   For example, if $a$ and $b$ are distributions on $\integers_+$
such that $b$ is a  Bernoulli probability distribution and $a_0 \geq b_0$,
then  $d_{TV}(a,b)= b_1 - a_1.$
\end{remark}

The proofs of  \eqref{eq:Y_tildeY_couple} and  \eqref{eq:tildeY_checkY_couple} are similar. Since the
proof of   \eqref{eq:Y_tildeY_couple} depends slightly on  \eqref{eq:tildeY_checkY_couple}, we prove
\eqref{eq:tildeY_checkY_couple} first.

 Fix $t \geq \tau$ and $y\geq m.$
The conditional distribution of the increment $\check Y_{t+1} - \check Y_t$
of the Markov process $\check Y$ given $\check Y_t=y$ can be identified as follows:
\begin{align*}
&\calL\left(\check Y_{t+1} - \check Y_t \bigg|   \check Y_t=y\right)    \\
&= \calL\left(Z_{\ln((t+1)/\tau)}  - Z_{\ln(t/\tau)} \bigg|   Z_{\ln t}=y\right)  \\
&= \calL\left(Z_{\ln\left(1+ \frac 1 t \right)}  \bigg|  Z_0 =y\right)  \\
& = \calL \left( {\sf negbinom}\left( y  , \left( 1 + \frac{1}{t} \right)^{-\vartheta} \right) -y \right)
\end{align*}
Hence, the following lemma is relevant, were $\epsilon$ represents $\frac 1 t.$
\begin{lemma}  \label{lmm:YZ_cont_coupling_h}
  Let $y$ be a positive integer and $\vartheta, \epsilon > 0.$  Then
\begin{align}   \label{eq:YZ_cont_coupling_h}
&d_{TV}\left(    {\sf negbinom}\left( y  , \left( 1 + \epsilon \right)^{-\vartheta}  \right) -y  , \Ber( \vartheta y \epsilon)                   \right) \nonumber \\
&\leq \frac{\epsilon^2} 2  \left( \vartheta y + \vartheta^2 (2y+1)y  \right).
\end{align}
\end{lemma}
\begin{proof}
The shifted negative binomial distribution assigns more probability mass to 0 than the Bernoulli distribution:
\begin{align*}
\prob{   {\sf negbinom}\left( y  , \left( 1 + \epsilon \right)^{-\vartheta}  \right) -y =  0     } \geq 1- \vartheta y \epsilon,
\end{align*}
or equivalently,
\begin{align*}
( 1 + \epsilon )^{-\vartheta y} \geq 1- \vartheta y \epsilon,
\end{align*}
as is readily proved by considering the derivative of each side with respect to $\epsilon$ for $\epsilon > 0.$
Therefore, by Remark \ref{rmk:on_variation_distance}(ii),  the total variation distance to be bounded is
given by the difference in probability mass at 1 for the two distributions.  In other words, if $\delta$
denotes the variational distance on the lefthand side of  \eqref{eq:YZ_cont_coupling_h}, then
\begin{align*}
\delta = \vartheta  y \epsilon   -  y (1+ \epsilon)^{- \vartheta y} \left( 1 - (1+\epsilon) ^{-\vartheta} \right).
\end{align*}
Note that $\delta = 0$ for $\epsilon = 0.$ Dividing through by $y$
and differentiating with respect to $\epsilon$ we find
\begin{align*}
\frac{ d  \delta}{y d\epsilon}  = \vartheta    +   \vartheta y(1+ \epsilon)^{- \vartheta y-1} 
-   \vartheta (y+1) (1+ \epsilon)^{-\vartheta (y+1) -1},
\end{align*}
and in particular the derivative at $\epsilon = 0$ is also zero.   Differentiating again yields:
\begin{align*}
&\frac{ d^2  \delta}{y (d\epsilon)^2}  = -  \vartheta  y(\vartheta y +1)(1+ \epsilon)^{- \vartheta y-2}   \\
&~~~~~~~~~~   +   \vartheta (y+1)  (\vartheta (y+1) +1 )(1+ \epsilon)^{-\vartheta (y+1)-2} \\
& \overset{(a)}{\leq} -\vartheta  y (\vartheta y +1 ) + \vartheta (y+1)(\vartheta (y+1) +1 )  \\
&= \vartheta + \vartheta^2 (2y+1)
\end{align*}
where to get inequality (a) we first multiply the lefthand side by $(1+\epsilon)^{\theta y +2}$
(thus increasing it) and then multiplying the second term on the lefthand side by $(1+\epsilon)^{2\theta},$
thus increasing the positive term further.  The lemma follows by twice integrating with respect to $\epsilon.$
\end{proof}

\begin{proof}[Proof of \eqref{eq:tildeY_checkY_couple}]
Let $n$ be a positive integer with $n\geq m.$  We appeal to Lemma \ref{lmm:YZ_cont_coupling_h}
to show that  $\tilde Y_{[\tau, T]}$ and $\check Y_{[\tau, T]}$ can be
coupled (i.e. constructed on the same probability space) such that the probability of
coupling failure before both processes reach state $n$ is bounded as follows:
\begin{align*}
\prob{ \tilde Y_{[\tau, T]} \wedge n \neq  \check Y_{[\tau, T]}\wedge n } 
& \leq  \sum_{t=\tau}^T  \frac{  \vartheta n  + \vartheta^2  (2n+1)n    }{2t^2}   \\
& \leq \frac{  \vartheta n  +  \vartheta^2  (2n+1)n    }{\tau-1} .
\end{align*}
The construction is done sequentially in time, starting with the process $\tilde Y$, letting $\check Y_{\tau} =m,$
and enlarging the probability space $\tilde Y$ is defined on in order to construct $\check Y_t$ for $\tau+1 \leq t \leq T$
on the same probability space.
For each time $t$ in the range $\tau \leq t \leq T-1,$  once the random variable $\check Y_t$ has been constructed, if the
coupling has been successful so far  (i.e. $\tilde Y_{[\tau, t]} = \check Y_{[\tau, t]}$) and if $\check Y_t \leq n-1$, we appeal to
Lemma  \ref{lmm:YZ_cont_coupling_h}  with $y=\check Y_t$ to show that the coupling can be continued to work
at time $t+1,$ with coupling error bounded above by Lemma \ref{lmm:YZ_cont_coupling_h}.

For this same pair of processes, it follows that
\begin{align*}
\prob{ \tilde Y_{[\tau, T]} \neq  \check Y_{[\tau, T]} }   
\leq \frac{  \vartheta n  + \vartheta^2 (2n+1)n    }{\tau-1}  + \prob{  \check Y_{T}  \geq n }
\end{align*}
Since $\check Y_{T}=Z_{\ln(T/\tau)}$,  the distribution of $\check Y_{T}$
is ${\sf negbinom}\left(m, \left( \frac  {\tau} T  \right)^{\vartheta}\right)$,
and the set of such distributions is tight under the limiting regime of the proposition.  In other words, 
$\lim_{n\to\infty}  \lim\sup_{\tau, T \to \infty}  \prob{  \check Y_{T}  \geq n } = 0$
under the assumption $T/\tau$ is bounded. The statement  \eqref{eq:tildeY_checkY_couple} follows.
\end{proof}

The proof of \eqref{eq:Y_tildeY_couple},  given next, is based on the following lemma.

\begin{lemma}   \label{lmm:coupling_lemma_Y}
Given a positive integer $y,$  $(\theta_{u,v}) \in \reals_{>0}^{r\times r}$, $\rho,$
and $(\theta^*_v: v\in [r]) \in \reals_{>0}^{r}$,
let $\theta_v = \sum_u \rho_u \theta_{u,v}.$ 
Suppose $t\geq 1$  and $v\in [r]$ such that  $\frac{ \theta_{u, v}y}  {mt} \leq 1$ for all $u$,
$\frac{ \theta_{v}y}  t \leq 1,$ and $\frac{ \theta^*_{v}y}  t \leq 1.$ 
 Then
\begin{align*}
&d_{TV}\left(  \sum_u \rho_u  {\sf binom}\left(m, \frac{ \theta_{u, v} y}  {mt}  \right)  ,   \Ber\left(\frac{ \theta^*_v y} t \right)  \right) \\
&~~~~~~~~~ \leq 
  \frac{\theta^2_{\max}y^2}{t^2}   +  \frac {|\theta_v -\theta^*_v|y} t ,
\end{align*}
where $\theta_{\max} =\max_{u,v} \theta_{u,v}.$
\end{lemma} 

\begin{proof}
By the triangle inequality,
\begin{align}
d_{TV}\left(  \sum_u \rho_u  {\sf binom}\left(m, \frac{  \theta_{u, v} y}  {mt}  \right)  , 
                \Ber\left(\frac{ \theta^*_v y} t \right)  \right)  ~~~~\nonumber \\
\leq d_{TV}\left(  \sum_u \rho_u  {\sf binom}\left(m, \frac{  \theta_{u, v}y}  {mt} \right)  ,  
                   \Ber\left(\frac{ \theta_v y} t \right)  \right)  \label{eq:term1}  \\
~~~+ d_{TV}\left(  \Ber\left(\frac{ \theta_v y} t \right)   ,  \Ber\left(\frac{ \theta^*_v y} t \right)  \right)  \label{eq:term2}
\end{align}
To bound the term on line \eqref{eq:term1}, we appeal to Remark \ref{rmk:on_variation_distance}(ii).  
Note that the probability masses at 0 for the two distributions inside $d_{TV}$ are ordered as:
\begin{align*}
&\sum_u \rho_u  {\sf binom}\left(m, \frac{  \theta_{u, v} y}  {mt}  \right)  \bigg|_0 = \sum_u \rho_u \left(   1- \frac{\theta_{u,v} y} {mt} \right)^m \\
&~~~~~~\geq  \sum_u \rho_u \left(   1- \frac{\theta_{u,v}y} t \right) = 1- \frac{\theta_v y} t  = \Ber\left(\frac{ \theta_v y} t \right)\bigg|_0
\end{align*}
So the term in \eqref{eq:term1} is the difference of the probability masses at 1:
\begin{align*}
& \frac{\theta_v y} t  -  \sum_u \rho_u  {\sf binom}\left(m, \frac{  \theta_{u, v} y}  {mt} \right)  \bigg|_1    \\
&= \frac{\theta_v y} t   -    \sum_u \rho_u \frac{ \theta_{u,v} y} t   \left(   1- \frac{\theta_{u,v} y}{mt} \right)^{m-1} \\
&=    \sum_u \rho_u \frac{\theta_{u,v} y} t \left( 1 -   \left(   1- \frac{\theta_{u,v} y} {mt} \right)^{m-1} \right)  \\
& \leq \sum_u \rho_u  \left( \frac{ \theta_{u,v} y} t   \right)^2  \leq  \frac{  \theta_{\max}^2 y^2} {t^2}
\end{align*}
The term on line \eqref{eq:term2} is equal to  $\frac {|\theta_v -\theta^*_v|y} t .$
\end{proof}

\begin{proof}[Proof of \eqref{eq:Y_tildeY_couple}]
Let $n$ be a positive integer with $n\geq m.$  
Since the entries of  $\beta$ are assumed to be strictly positive, there is a finite
value $\theta_{\max}$ such that  $\theta_{u,v,t} \leq \theta_{\max}$ for all $t.$ 
Given $\epsilon > 0$ let $F$ be the event  defined by
$F = \{ |\theta_{v,t} -\theta^*_v| > \epsilon~ \mbox{\rm for some} ~ t\geq \tau \}.$
We appeal to Lemma \ref{lmm:coupling_lemma_Y}
to show that  $Y_{[\tau, T]}$ and $\tilde Y_{[\tau, T]}$ can be
coupled (i.e. constructed on the same probability space) such that the probability of
coupling failure before both processes reach state $n$ is bounded as follows:
\begin{align*}
&\prob{ Y_{[\tau, T]} \wedge n \neq  \tilde Y_{[\tau, T]}\wedge n } \\
&  \leq \prob{F} +
 \sum_{t=\tau}^{T-1}  \left(  \frac{ \theta^2_{\max}n^2}{t^2}   +   \frac {|\theta_{v,t} -\theta^*_v| n} t  \right) \\
&\leq  \mathbb{P}(F)  +  \frac{ \theta^2_{\max}n^2}{\tau-1}  +   \epsilon n \ln \frac{T}{\tau-1}
\end{align*}
For this same pair of processes, it follows that
\begin{align*}
&\prob{ Y_{[\tau, T]} \neq  \tilde Y_{[\tau, T]} }    & \\
&\leq  \prob{F} +  \frac{  \vartheta n  + \vartheta^2 (2n+1)n    }{\tau-1}  + \prob{  \tilde  Y_{T}  \geq n }
\end{align*}
By Proposition 2  (almost sure convergence of $\eta_t \to \eta^*$) $\mathbb{P}(F) \to 0$ as $\tau\to\infty.$
By  \eqref{eq:tildeY_checkY_couple},  already proved,
$$
\lim_{n\to\infty}  \lim\sup_{\tau, T \to \infty}  
\bigg| \prob{  \tilde Y_{T}  \geq n }  - \prob{  \check Y_{T}  \geq n } \bigg| =0,
$$
so, just as for $\check Y_T$,  the set of distributions of $\tilde Y_T$ is tight under the limiting regime
of the proposition.  In other words, 
$\lim_{n\to\infty}  \lim\sup_{\tau, T \to \infty}  \prob{  \tilde Y_{T}  \geq n } = 0$
under the assumption $T/\tau$ is bounded. The statement  \eqref{eq:Y_tildeY_couple} follows.
\end{proof}

%%%%%%%%%%%%%%%%%
\section{Proof of Proposition \ref{prop:YJ_couple}}   \label{app:proof_YJ_couple}
The proof is similar to the proof of  Proposition \ref{prop:Y_couple}.
Before proving the proposition we introduce some notation and present
a lemma that is used to bound the coupling failure probability at a given
step in the construction.
%Some notation:
 A {\em subprobability vector} for a set $[d]= \{1, \ldots , d\}$ is a $d$-tuple of the form
$\underline a = (a_i : i\in [d])$ such that $a_i\geq 0$ for $i\in [d]$ and $\sum_{i\in [d]} a_i \leq 1.$
Let $r$ and $J$ be positive integers.    Suppose $\rho$ is a probability distribution on $[J].$  
Suppose $\underline p,  \underline p',$ and $\underline q_{u,\cdot}$ for all $u\in [r]$  are
subprobability vectors for $[J].$
\begin{itemize}
\item   Let $\sel (\underline p )$ represent the {\em selector} distribution on $\integers^J_+$
with probability mass $p_j$ on the vector $e_j$, and probability
mass $1-\sum_j p_j$ on the zero vector.

\item Let $\sel^{*m}(\underline p)$ denote the distribution of the sum of $m$ independent
random vectors, each with the distribution $\sel (\underline p).$   In other words,
$\sel^{*m}(\underline p)$ is the $m$-fold convolution of $\sel(\underline p).$   

\item  Let $\sum_u \rho_u \sel^{*m}(\underline q_{u,\cdot})$ denote the distribution
that is a mixture of the distributions $\sel^{*m}(\underline q_{u,\cdot})$ as $u$ varies
with selection probability distribution $\rho.$

\item Let $\otimes_{j=1}^J \Ber(p_j)$  denote the distribution of a random $J$ vector
with independent coordinates, with coordinate $j$ having distribution $\Ber(p_j).$
\end{itemize}

\begin{lemma}  \label{lmm:coupling_lemma_YJ}
 Suppose $\rho$ is a probability distribution on $[J].$  
Suppose $\underline p,  \underline p',$ and $\underline q_{u,\cdot}$ for all $u\in [r]$  are
subprobability  vectors for $[J].$
\begin{align}
&d_{TV}\left(   \otimes_{j=1}^J \Ber(p_j)  , \sel (\underline p) \right)  \leq  \left( \sum_{j \in [J]} p_j \right)^2   \label{eq:Ber_select} \\
&d_{TV}( \sel(\underline p), \sel(\underline p')) \leq \sum_{j\in [J]} | p_j -p'_j |   \label{eq:two_select}    \\
&d_{TV}\left( \sel\left(\sum_u \rho_u \underline q_{u,\cdot} \right),
 \sum_{u\in [r]}  \rho_u \sel^{*m}\left(\frac 1 m \underline q_{u,\cdot}\right)  \right)  \nonumber  \\
&~~~\leq  \sum_{u\in [r]} \rho_u \left( \sum_{j\in [J]} q_{u,j} \right)^2   \label{eq:mixed_select}
\end{align}
\end{lemma}
\begin{proof}
Inequality \eqref{eq:two_select} follows easily from the definitions.   The proofs of the other two inequalities rely on
Remark \ref{rmk:on_variation_distance}(ii).    Note that the distribution $\sel (\underline p)$ is supported on $J+1$
points in $\integers^J,$  namely, $\underline 0, e_1, \ldots , e_J.$  Also,
$$
 \otimes_{j=1}^J \Ber(p_j) \bigg|_{\underline 0} = \prod_{j\in [J]} (1-p_j)  \geq 1 -\sum_{j\in [J]} p_j = \sel(p)\bigg|_{\underline 0}.
$$
Thus, by Remark \ref{rmk:on_variation_distance}(ii),
\begin{align*}
&d_{TV}\left(   \otimes_{j=1}^J \Ber(p_j)  , \sel (\underline p) \right)   =
\sum_{j \in [J]}  p_j \left[  1 - \prod_{j'\in [J], j'\neq j} (1-p_{j'}) \right]  \\
& \leq  \sum_{j \in [J]}  p_j \sum_{j'\in [J], j'\neq j} p_{j'}  \leq  \left( \sum_{j\in [J]} p_j \right)^2,
\end{align*}
which establishes \eqref{eq:Ber_select}.   The proof of \eqref{eq:mixed_select}, given next, is similar.
The probability masses
the two distributions on the lefthand side of \eqref{eq:mixed_select} place at zero is ordered as follows:
\begin{align*}
\sum_{u\in [r]} \rho_u  \left( 1 -\frac 1 m  \sum_{j\in [J]} q_{u,j}\right)^m  
  \geq 1 - m \sum_{u\in [r]} \rho_u \sum_{j\in [J]} q_{u,j}.
\end{align*}
Therefore, by Remark \ref{rmk:on_variation_distance}(ii),
\begin{align*}
&d_{TV}\left(\sum_{u\in [r]}  \rho_u \sel^{*m}\left(\frac 1 m \underline q_{u,\cdot}\right) , \sel\left(\sum_u \rho_u \underline q_{u,\cdot} \right)\right)  \\
&~~~=  \sum_{j \in [J]} \sum_{u\in[r]} \rho_uq_{u,j}  \left[  1 - \left( 1 - \frac 1 m \sum_{j'\in [J]: j\neq j} q_{u,j'}\right)^{m-1}\right]  \\
&~~~ \leq  \sum_{j\in [J]} \sum_{u\in[r]} \rho_uq_{u,j}  \left[  \sum_{j' \in [J]} q_{u,j'}  \right]   \\
& ~~~\leq  \sum_{u\in [r]} \rho_u \left( \sum_{j\in [J]} q_{u,j} \right)^2,
\end{align*}
which establishes \eqref{eq:mixed_select}.
\end{proof}

\begin{lemma}  \label{lmm:combined}
Suppose the conditions of Lemma \ref{lmm:coupling_lemma_YJ} hold, and, in addition, $\sum_u \rho_u q_{u,j} = p_j'$
for all $j\in [J].$   Then
\begin{align}   \label{eq:combined}
&d_{TV}\left(   \otimes_{j=1}^J \Ber(p_j)  ,  \sum_{u\in [r]}  \rho_u \sel^{*m}\left(\frac 1 m \underline q_{u,\cdot}\right)  \right)   \\
& \leq   \left( \sum_{j \in [J]} p_j \right)^2    +   \sum_{j\in [J]} | p_j -p'_j |   +
 \sum_{u\in [r]} \rho_u \left( \sum_{j\in [J]} q_{u,j} \right)^2  .  \nonumber
\end{align}
\end{lemma}
\begin{proof}
The lefthand side of \eqref{eq:combined} is less than or equal to the sum of the
lefthand sides of  \eqref{eq:Ber_select}-\eqref{eq:mixed_select} by the triangle inequality
for $d_{TV}.$   The righthand side of \eqref{eq:combined} is the sum of the righthand
sides of  \eqref{eq:Ber_select}-\eqref{eq:mixed_select}.   So the lemma follows from
Lemma \ref{lmm:coupling_lemma_YJ}.
\end{proof}

\begin{proof}[Proof of Proposition \ref{prop:YJ_couple}]
By the tightness of   $\calL(Y^j_T | \ell_\tau=v_j)$ for each $j$ in the limit regime of 
the proposition, implied by Proposition \ref{prop:Y_couple} and the known distribution of $\check Y_T,$
 it suffices to prove the proposition with $Y^{[J]}_{[1,T]}$ and
$\tilde Y^{[J]}_{[1,T]}$ each replaced by versions of the same processes that are
stopped when the sum of the vertex degrees (i,e. the coordinates) of the process
first becomes greater than or equal to a fixed, positive integer $n.$   So let $n$ be a fixed, positive integer.
Let the process $Y^{[J]}_{[1,T]}$ be given, defined on some probability space.   By enlarging
the probability  space, we can construct $\tilde Y^{[J]}_{[1,T]}$ on the same space, and the total
variation distance is upper bounded by the probability the processes are different from each other
at some time before the sum of coordinates is greater than $n$ or before time $T+1.$
The construction is done sequentially in time.   For each time $t$ in the range $1 \leq t \leq T-1,$
once the random variable $\tilde Y_t$ is constructed, we appeal to Lemma \ref{lmm:coupling_lemma_YJ}
with $q_{u,j}$ in the lemma given by $\frac{ \theta_{u,v_j,t} \tilde Y^j_t} t .$    Since the entries of  $\beta$
are assumed to be strictly positive, there is a finite value $\theta_{\max}$ such that
$\theta_{u,v,t} \leq \theta_{\max}$ for all $t.$   Given $\epsilon > 0$ let $F$ be the event
defined by   $F = \{ |\theta_{v,t} -\theta^*_v| > \epsilon~ \mbox{\rm for some} ~ t\geq \tau \}.$

Fix $t$ with $\tau_1  \leq t \leq T.$   Let $A_t = \{j : \tau_j \leq t\},$  so that $A_t$ is the set of
vertices in $[J]$ that are active at time $t.$    For $j \not\in A_t$  the values of $Y^j_{t+1}$ and
$\tilde Y^j_{t+1}$ are deterministic and they are equal.

If $t+1 = \tau_j$ for some $j$ we call $t$ an {\em exceptional} time. Exceptional times must
be handled differently than other times because for such a time, conditioning on
$(\ell_{\tau_j} = v_j),$ or, equivalently, on $(\ell_{t+1} = v_j),$
effects the distribution of  $(Y^{j'}_{t+1} -Y ^{j'}_t : j'\in A_t),$
and Lemma \ref{lmm:combined}  doesn't apply.
The effect of such exceptional times on coupling error can be bounded as follows. First,
there are less than or equal to $J$ exceptional times.   Secondly, for such an exceptional
time $t$, 
$$
\prob{
Y^{j'}_{t+1} -Y ^{j'}_t  \neq 0 ~ \mbox{for some}~  j' \in A_t |  \ell_{\tau^j} =v_j   }   \leq  \frac{n\theta_{\max} }{t}
$$
and also
$$
\prob{
\tilde Y^{j'}_{t+1} -\tilde  Y ^{j'}_t  \neq 0 ~ \mbox{for some}~  j' \in A_t   }   \leq  \frac{n\theta_{\max} }{t}
$$
so that if $Y^{[J]}$ and $\tilde Y^{[J]}$ are coupled up to time $t$,  the coupling can be extended to to time $t+1$
with additional probability of coupling error at most $\frac{n\theta_{\max}}{t}.$
The overall increase in the probability of coupling failure due to the
exceptional times is less than or equal to  $\frac{ \theta_{\max} n J}{\tau_0} \to 0.$

Next, suppose $t$ is not an exceptional time.
Let $y \in \integers_+^J$ such that $\sum_j  y^j \leq n$ and $y^j=0$ for $j\not\in A_t.$
Lemma \ref{lmm:combined} with $p_j=\frac{\vartheta_j y^j} t,$
$p'_j = \frac{\theta_{v_j,t} y^j} t,$ and $q_{u,j} = \frac{\theta_{u,v_j,t} y^j} t$
for $j\in A_t$  implies that the error for attempting to couple $\tilde Y^{[J]}_{t+1}$ to $Y^{[J]}_{t+1}$
given  $\tilde Y^{[J]}_t = Y^{[J]}_{t}= y^{[J]}$ is less than or equal to
\begin{align*}
&\frac{\left(  \sum_{j\in[J]} \vartheta_j y^j  \right)^2}{t^2}
+ \frac{ \sum_{j\in [J]}  |\vartheta_j - \theta_{v_j,t} | y^j} t   \\
&~~~~~~~~~~+ \sum_{u\in [r]}  \rho_u  \frac{\left(  \sum_{j\in [J]} \theta_{u,v_j,t} y^j\right)^2}{t^2}  
\end{align*}
Hence, the probability of coupling failure, before the sum of degrees is $n$ and before time $T+1$, is
less than or equal to
\begin{align*}
\mathbb{P}(F) +  \frac{ \theta_{\max} n J}{\tau_0} +  \sum_{t=\tau_0}^T 
\left(   \frac{J^2 \theta_{\max}^2 n^2}{t^2}  + \frac{n\epsilon} t + \frac{\theta_{\max}^2n^2}{t^2}  \right),
\end{align*}
which can be made arbitrarily small as in the proof of Proposition \ref{prop:Y_couple}.
\end{proof}

%%%%%%%%%%%%%%%%%%%%%
\section{Appendix: Alternative proof of  Proposition \ref{prop:empirical_m1}}  \label{app:empirical_m1}

This section gives an alternative proof of Proposition \ref{prop:empirical_m1},
but only for convergence in probability, based on Corollary \ref{cor:Y_vs_Z_mult}.
The same method can be used to prove Proposition \ref{prop:error_scaling}(b),
concerning the convergence in probability of the fraction of label errors
made by two recovery algorithms.  We use the notation
given just before the statement  of  Proposition \ref{prop:empirical_m1}.

Since the labels of the vertices are independent with distribution $\rho,$  by the
law of large numbers,
\begin{align*}
\lim_{T\to\infty}  \frac{H^v(T)}{T} = \rho_v   ~~~~\mbox{ (a.s. and in probability) .}
\end{align*}
Thus, it suffices to show that for fixed $n\geq m,$
\begin{align*}
\lim_{T\to\infty}  \frac{N_n^v(T)}{T} = \rho_v    p_n( \theta^*_v,m)~~~~\mbox{ (in probability). }
\end{align*}
By the Chebychev inequality,
for that it suffices to show the following two conditions:
\begin{align}
\lim_{T\to\infty}  \frac{\expect{N^v_n(T)}} T  &  = \rho_v p_n(\theta^*_u,m)   \label{eq:mean_conv}  \\
\lim_{T\to\infty}  \var \left( \frac{ {N^v_n(T)} }T \right) & = 0.  \label{eq:var_conv}
\end{align}
Write $N_n^v(T) = \sum_{\tau=1}^T  \chi_\tau,$ where $\chi_{\tau}=1$ if  $\ell_{\tau}=v$ and
the degree of vertex $\tau$ at time $T$ is $n,$  and $\chi_{\tau}=0$ otherwise. 
Then $\vert \expect{N^v_n(T)} - \sum_{\tau=t_o+1}^T   \expect{ \chi_{\tau}} \vert \leq t_o.$
By Corollary \ref{cor:Y_vs_Z_mult} with $J=1,$  $t=T$,  and $v\in [r],$
\begin{align}
&\lim_{\tau_0 \to\infty}  \sup_{\tau, T :  \tau > \tau_0 ~\mbox{\small and} ~ T > \tau_0}  \nonumber \\
&\bigg|   \expect{ \chi_{\tau}} - \rho_v  \pi_n\left(\ln(T/\tau), \theta^*_v, m\right)   \bigg|  =  0.  \label{eq:first_moment}
\end{align}
Therefore, by the bounded convergence theorem,   \eqref{eq:mean_conv} holds with
\begin{align}
p_n( \theta, m) & =  \frac 1 T \int_0^T   \pi_n(\ln (T/t), \theta, m)  dt  \nonumber  \\
&  \overset{(a)}{=} \binom{n-1}{m-1} \int_0^1 u^{m\theta}(1-u^{\theta})^{n-m}  du   \nonumber  \\
&  \overset{(b)}{=} \frac{1}{\theta}  \binom{n-1}{m-1}  \int_0^1 v^{m -1 + \frac 1 {\theta}} (1-v)^{n-m}  dv \nonumber  \\
& \overset{(c)} {=}   \frac 1 {\theta}   \binom{n-1}{m-1}  B\left( m + \frac 1 {\theta} , n-m+1 \right)   \nonumber  \\
&  \overset{(d)}{=}  \frac{  \Gamma\left(\frac 1 \theta  + m  \right)  \Gamma(n)  }  
  {\theta \Gamma(m) \Gamma\left( n+ \frac 1 \theta + 1   \right)  },  \label{eq:marginal_dist}
\end{align}
where (a) follows by the definition of the negative binomial distribution and change of
variable $u=t/T,$  (b) follows by the change of variable $v=u^{\theta}$, and (c) and (d) follow from
standard formulas for the beta function, $B.$

It remains to verify  \eqref{eq:var_conv}.   First note that
\begin{align}   \label{eq:Cov_expand}
\var (N_n^v(T) ) = \sum_{\tau_1=1}^T  \sum_{\tau_2=1}^T  \Cov (\chi_{\tau_1} , \chi_{\tau_2} ).
\end{align}
Note that 
$$
\expect{\chi_{\tau_1}\chi_{\tau_2} } = \rho_v^2  
 \prob{ Y^1_T=n,   Y^2_T=n \bigg|  \ell_{\tau_1} =  \ell_{\tau_2} = v  },
$$
and by Corollary \ref{cor:Y_vs_Z_mult}  with $J=2,$  $t=T$,  and $v_1=v_2=v,$
\begin{align*}
& \lim_{\tau_0 \to\infty}  \sup_{\tau_1, \tau_2, T: \tau_0 \leq \tau_1 < \tau_2 ~\mbox{\small and} ~ T\geq \tau_0}   \\
& \bigg|   \expect{\chi_{\tau_1}\chi_{\tau_2} }  - \rho_v^2   \pi_{n}\left(\ln \frac T {\tau_1}  ,\theta^*_v, m \right)
\pi_{n}\left(\ln \frac T {\tau_2}  ,\theta^*_v, m \right) \bigg| \to 0 .
\end{align*}
So, in view of \eqref{eq:first_moment} and the fact
$\Cov(\chi_{\tau_1},\chi_{\tau_2}) = \expect{\chi_{\tau_1}\chi_{\tau_2} }  -   \expect{ \chi_{\tau_1} }  \expect{\chi_{\tau_2} } ,$
 \begin{align*}
\lim_{\tau_0 \to\infty}  \sup_{\tau_1, \tau_2, T: \tau_0 \leq \tau_1 < \tau_2 ~\mbox{\small and} ~ T\geq \tau_0} 
 |\Cov(\chi_{\tau_1},
\chi_{\tau_2}) | = 0.
\end{align*}
Using this to bound the terms on the righthand side of \eqref{eq:Cov_expand} with $\tau_1, \tau_2 \in [\tau_0, T]$
and $\tau_1 \neq \tau_2,$   and bounding the other terms by one, yields:
\begin{align*}   
&\var (N_n^v(T) )  \leq   2T\tau_0 + T + \\
& T^2 \left( \sup_{\tau_1, \tau_2, T: \tau_0 \leq \tau_1 < \tau_2 \leq T} |\Cov(\chi_{\tau_1},  \chi_{\tau_2}) | \right) =o(T^2).
\end{align*}
if $T, \tau_0 \to\infty$ with $\tau_0/T \to  0.$    This implies \eqref{eq:var_conv}, completing the
alternative proof of the Proposition \ref{prop:empirical_m1} (for convergence in probability).

\begin{remark}
In essence, the calculation in \eqref{eq:marginal_dist} demonstrates that the limiting empirical distribution of
degree for vertices of a given label $v$ at a large time $T$, is the marginal distribution for the following joint distribution: 
the vertex time of arrival is uniform over $[0,T]$ and, given the arrival is at time $\tau$, the conditional distribution
of degree is  ${\sf negbinom}\left(m,  \left( \frac {\tau}   T \right)^{\theta^*_v} \right).$
\end{remark}

%%%%%%%%%%%%%%%%%%%

\section{Consistent estimation of the growth rate parameter for a given vertex}
\label{app:consistency}

Proposition \ref{prop:Y_large_time} is proved in this section and evidence
for Conjecture \ref{conj:Y_large_time_sharp} is given.
First a different method for estimating the rate parameter of $Y$ is
established.
Consider the \BA model with communities.
Fix $\tau_o\geq 1$ and $\tau$ with $\tau \geq \max\{\tau_o, t_o\}$
(recall that $t_o$ is the number of vertices in the initial graph).
Let $Y_t$ denote the degree of $\tau_o$ in $G_t$  for all $t\geq \tau.$
To avoid triviality associated with an isolated vertex in $G_{t_o}$,
suppose $Y_\tau \geq 1.$   We also suppose $Y_\tau \leq m\tau$, so by
induction on $t$,  $\frac{Y_t} t \leq m$ for all $t\geq \tau.$
Let $\vartheta = \theta^*_v$ where $v$ is the label of $\tau_o.$
\begin{proposition}  (Consistent estimation of rate parameter)
  \label{prop:consistent_rate_est} 
The estimator $\hat \vartheta_T$ defined by
\begin{align}  \label{eq:theta_estimator}
\hat \vartheta_T = \frac{Y_T - Y_\tau}{\sum_{t=\tau}^{T-1} \frac{Y_t} t }
\end{align}
is consistent.  In other words,  $\lim_{T\to\infty}  \hat \vartheta_T = \vartheta$ a.s.
\end{proposition}

To prove the proposition we first examine a sequential version of
 $\hat \vartheta_T.$  Given a positive constant $M$ with $M > m,$  let
$T_M$ denote the stopping time defined by
$$
T_M = \min \left\{ T\geq \tau  :   \sum_{t=\tau}^{T_M} \frac{Y_t}{t} \geq M  \right\}
$$
Let  $\hat{\hat \vartheta}_M$   be $\hat \vartheta_T$ for $T=T_M,$  or, in other words,
$$
\hat{\hat \vartheta}_M = \frac{Y_{T_M} - Y_\tau}{\sum_{t=\tau}^{T_M-1} \frac{Y_t} t }.
$$

\begin{lemma}    \label{lemma:stopped_estimator}
Under the idealized assumption $\eta_t \equiv \eta^*,$
for any $\epsilon > 0$,
\begin{align*}
\prob{  \bigg|  \hat{\hat \vartheta}_M   -  \vartheta \bigg| \geq \epsilon }
\leq \frac{m \vartheta M}{\epsilon^2(M-m)^2.}
\end{align*}
\end{lemma}

\begin{proof}
Notice that
the denominator of $\hat{\hat \vartheta}_M $ is in the interval $[M-m,M]$
with probability one.  Also,
\begin{align*}
Y_{T} - Y_\tau - \vartheta\left( \sum_{t=\tau}^{T-1} \frac{Y_t} t  \right)
& =  \sum_{t=\tau}^{T-1}   \left( Y_{t +1}-Y_t - \frac{\vartheta Y_t} t  \right),
\end{align*}
so that  $\left(Y_{T} - Y_\tau - \vartheta \sum_{t=\tau}^{T-1} \frac{Y_t} t : T\geq \tau\right) $
is a martingale.   Since $T_M$ is a bounded optional sampling time, the
martingale optional sampling theorem can be applied to yield
\begin{align*}
\expect{  Y_{T_M} - Y_\tau } = \expect{ \vartheta \sum_{t=\tau}^{T_M-1} \frac{Y_t} t  }
\in [\vartheta (M-m), \vartheta M] .
\end{align*}
Next we bound the second moments.    It is easy to show that
a random variable $U$ with values in $[0,m]$ and mean
$\mu$ satisfies
$\var(U) \leq m^2 \frac{\mu} m  \left(1 - \frac{\mu} m\right) \leq m\mu.$
For any $t\geq \tau,$   $Y_{t +1}-Y_t $ takes values in $[0,m]$ and,  given
the past $\calF_t$,  it has conditional mean $\frac{\vartheta Y_t} t. $
It follows that
$\expect{\left( Y_{t +1}-Y_t - \frac{\vartheta Y_t} t \right)^2  \bigg| \calF_t }
\leq \frac{m\vartheta Y_t}{t}.$
Therefore, again using the optional sampling theorem,
\begin{align*}
&\expect{  \left(Y_{T_M} - Y_\tau - \vartheta \left( \sum_{t=\tau}^{T_M-1} \frac{Y_t} t \right) \right)^2 } \\
&= \expect{  \sum_{t=\tau}^{T_M-1} 
 \expect{\left( Y_{t +1}-Y_t - \frac{\vartheta Y_t} t \right)^2  \bigg| \calF_t } }  \\
 & \leq m \vartheta \expect{\sum_{t=\tau}^{T_M-1}   \frac{Y_t} t   } \leq m\vartheta M
\end{align*}
Thus, for any $\epsilon > 0,$ the Chebychev
inequality yields
\begin{align*}
&\prob{ \bigg|
Y_{T_M} - Y_\tau - \vartheta \left( \sum_{t=\tau}^{T_M-1} \frac{Y_t} t \right)
\bigg| \geq  \epsilon (M-m)}   \\
&~~~~~~~~~~~~~~~~~~\leq   \frac{m \vartheta M}{\epsilon^2(M-m)^2},
\end{align*}
which implies the conclusion of the proposition.
\end{proof}

\begin{proof}[Proof of Proposition \ref{prop:consistent_rate_est}]
Since $Y_t \geq 1$ for all $t\geq \tau,$
 $\sum_{t=\tau}^{T-1} \frac{Y_t} t  \to \infty$ a.s. as $T\to\infty.$
 Therefore, for $\tau_o$ fixed  (the vertex for which we want to estimate
 the rate parameter),  whether $\hat \vartheta$ is consistent does not depend
 on the choice of $\tau.$    For any given $\epsilon > 0,$
by taking $\tau$ very large, we can thus ensure  $|\eta_t - \eta^*| \leq \epsilon$
for all $t\geq \tau$ with probability at least $1-\epsilon.$
Therefore, it suffices to prove the proposition under the added assumption
$\eta_t \equiv \eta^*$ for all $t \geq \tau.$  It follows that it suffices
to prove that $\hat{\hat \vartheta}_M $
is a consistent family of estimators of $\vartheta.$

So it remains to prove consistency of the family of estimators
$\hat{\hat \vartheta}_M $  as $M\to \infty.$ 
For that purpose, it suffices to show that for arbitrarily
small $\epsilon > 0$,   along the sequence of $M$ values
$M_k = (1+\epsilon)^k$,  the estimation error is greater
than or equal to $\epsilon$ for only finitely many values
of $k,$  with probability one.     That follows from Lemma
\ref{lemma:stopped_estimator},  because the error probability 
in Lemma \ref{lemma:stopped_estimator} is $O(1/M)$ and
$\sum_{k=1}^{\infty}  1/M_k  < \infty,$ so the Borel Cantelli
lemma implies the desired conclusion.
\end{proof}

Proposition \ref{prop:Y_large_time} will follows from 
Proposition \ref{prop:consistent_rate_est} and the following lemmas,
which are essentially Gr\"{o}nwall type inequalities.
\begin{lemma}  \label{eq:f_lemma}
Suppose $(f(s) : s \in \reals_+)$ is a positive nondecreasing function
such that for some $\vartheta > 0,$
\begin{align*}
\lim_{S \to \infty} \frac{f(S)-f(0)}{\int_0^S  f(u) du} = \vartheta .
\end{align*}
Then
\begin{align*}
\lim_{S \to \infty} \frac{\ln f(S)}{S} = \vartheta .
\end{align*}
\end{lemma}

\begin{proof} 
Given any $\epsilon > 0$, there exits $S_\epsilon $ such that
\begin{align*}
f(S) \geq  f(0) + (\theta-\epsilon) \int_0^S  f(u) du~~~\mbox{for } S\geq S_{\epsilon}.
\end{align*}
Since $f(u)\geq f(0)$ for all $u,$
\begin{align*}
f(S) \geq  C +  (\theta-\epsilon)  \int_{S_{\epsilon}}^S  f(u) du ~~~\mbox{for } S\geq S_{\epsilon}
\end{align*}
where $C= f(0)(1 + (\theta-\epsilon) S_\epsilon).$   Thus, for any $s \geq 0$,
setting $S=s+ S_{\epsilon},$  yields
\begin{align*}
f(s+S_{\epsilon}) \geq  C + (\theta-\epsilon)   \int_0^s  f(s+S_{\epsilon}) du ~~~\mbox{for } s\geq 0.
\end{align*}
By induction on $k$ it follows that $f(s+S_{\epsilon}) 
 \geq C \sum_{j=0}^k \frac{((\vartheta-\epsilon)s)^j}{j!},$ so that
 $f(s+S_{\epsilon})   \geq C\eexp^{s(\vartheta-\epsilon)}$  for all $s\geq 0.$
 Therefore,  $\lim\inf_{S \to \infty} \frac{\ln f(S)}{S} \geq \vartheta.$
It can be proved similarly that $\lim\sup_{S \to \infty} \frac{\ln f(S)}{S} \leq \vartheta,$  establishing the lemma.
 \end{proof}
 
 \begin{lemma}  \label{eq:lemma_Ylim}
 Let $(y_t : t\in\{\tau, \tau+1, \ldots  \})$ be a sequence of positive numbers such that
 $y_{t+1}-y_t \in [0,m]$ for all $t\geq \tau,$   and such that
 \begin{align*}
\lim_{T \to \infty} \frac{y_T- y_{\tau}}{\sum_{t=\tau}^{T-1}  \frac{y_t} t }= \vartheta
\end{align*}
Then
\begin{align*}
\lim_{T \to \infty} \frac{\ln y_T}{\ln (T/\tau)} = \vartheta
\end{align*}
\end{lemma}
\begin{proof}
We shall apply the previous lemma by switching to a continuous parameter and then
applying a change of time.  
Note that $0 \leq \frac 1 t - \int_t^{t+1} \frac 1 s ds = \frac 1 t - \ln(1 + \frac 1 t) \leq \frac 1 {2t^2}.$
Hence
\begin{align*}
0 \leq  \sum_{t=\tau}^{T-1}  \frac{y_t} t    -   \int_\tau^T  \frac{y_{\lfloor t \rfloor}} t dt
 \leq \frac 1 2 \sum_{t=\tau}^{T-1}  \frac{y_t} {t^2}  = o\left(  \sum_{t=\tau}^{T-1}  \frac{y_t} t     \right)  .
\end{align*}
The hypotheses thus imply
 \begin{align*}
\lim_{T \to \infty} \frac{y_T- y_{\tau}}{\int_\tau^T  \frac{y_{\lfloor t \rfloor}} t dt}= \vartheta.
\end{align*}
Letting $f(s)=y_{\lfloor \tau \eexp^s \rfloor},$ the change of variable $u=\ln(t/\tau)$ yields
\begin{align*}
\frac{y_T- y_{\tau}}{\int_\tau^T  \frac{y_{\lfloor t \rfloor}} t dt}  
=\frac{f(\ln(T/\tau) )- f(0)}{\int_\tau^T  \frac{f(\ln(t/\tau))} t dt} 
=  \frac{f({\ln(T/\tau)})-f(0)}{\int_0^{\ln(T/\tau)}  f(u) du},
\end{align*}
so the hypotheses of Lemma \ref{eq:f_lemma} hold.  Lemma \ref{eq:f_lemma} yields
\begin{align*}
\lim_{S \to \infty} \frac{\ln y_{\lfloor \tau \eexp^S \rfloor}}{S} = \vartheta,
\end{align*}
which by the change of variable $S=\ln(T/\tau),$  is equivalent to the conclusion of the lemma.
\end{proof}

\begin{proof}[Proof of Proposition \ref{prop:consistent_rate_est}]
Proposition \ref{prop:consistent_rate_est} follows directly from
Proposition  \ref{prop:consistent_rate_est}  and Lemma \ref{eq:lemma_Ylim}.
\end{proof}

{\em Evidence for Conjecture \ref{conj:Y_large_time_sharp}}
The Kesten-Stigum theorem \cite{KestenStigum66}
in the case of single-type branching processes implies
that  $\lim_{s\to\infty} Z_s \eexp^{-\vartheta s} = W$  a.s. for some random variable
$W$ such that $\prob{W> 0}=1$ and $\expect{W}=Z_0=m.$    (This follows from the fact that
$Z$ restricted to multiples of any small positive constant $h > 0$ is
a discrete-time single-type Galton Watson branching process with
number of offspring per individual per time period, represented by
a random variable $L_h$, such that $L_h$ has the ${\sf negbinom}(m,\eexp^{\vartheta h})$
distribution.  Note that $\prob{L_h \geq 1}=1$ and
$\expect{L_h \ln L_h} < \infty.$)   Since $Z_t \eexp^{-\vartheta s}$ also converges
in distribution to the Gamma distribution with parameters $m$ and $\vartheta,$
it follows that $W$ has such distribution.  It follows that
\eqref{eq:Y_conjecture} holds if the process $Y$ is replaced by the process $\check Y.$

%%%%%%%%%%%%%%%%%%%
\section{Proof of Proposition \ref{prop:joint_ZA_transform}}   \label{app:proof_joint_ZA}
The process $Z$ with parameters $\lambda, m$  represents the total population of a branching process starting with $m$ root individuals
at time 0, such that each individual in the population spawns new individuals at rate $\lambda.$
 And $A_{s}$ represents the sum of the lifetimes, truncated at time $s$, of all the individuals in the population.
 The joint distribution of $(Z,A)$ with parameters $\lambda, m$ is the same as the distribution of the sum of $m$ independent
 versions of $(Z,A)$ with parameters $\lambda, 1,$   Hence, it suffices to prove the lemma for $m=1.$   
 
 So for the  remainder of this proof suppose $m=1$; there is a single root individual.
Suppose there are $n(s)$ children of the root individual, produced at times $R_1, \ldots, R_{n(s)}$.
Then
\begin{align}
Z_{s}  & = 1 + \sum_{l  = 1}^{n(s)} Z^\ell_{s - R_l}   \\
A_{s}  & = s + \sum_{l = 1}^{n(s)} A^\ell_{s - R_l}
\end{align}
where $Z^\ell_{s-R_l}$ denotes the total subpopulation of the $l^{th}$ child of the root, $s-R_l$ time units after the birth
of the $l^{th}$ child, and $A^\ell_{s-R_l}$ is the associated sum of lifetimes of that subpopulation, truncated $s-R_l$ time
units after the birth of the $l^{th}$ child (i.e. truncated at time $s$).   The processes $(Z^l, A^l)$ are independent
and have the same distribution as $(Z,A).$   The variables $R_1, \ldots  , R_{n(s)}$ are the points of a Poisson process of rate $\lambda.$
Therefore,
\begin{align*}
    e^{uZ_{s} + vA_{s}} &= \eexp^{u + vs}\prod_{l = 1}^{n(s)} \exp( uZ_{s-R_l} + vA_{s - R_l}),
\end{align*}
which after taking expectations yields
\begin{align*}
    \psi_{\lambda,1}(u, v, s) & =  \eexp^{u + vs} \expectLone{ \prod_{l = 1}^{n(s)} \exp(uZ^l_{s-R_l} + vA^l_{s - R_l}) } .
\end{align*}
Since $n(s)$ is a Poisson$(\lambda)$ random variable, and, given $n(s)$,
$R_1, \ldots, R_{n(s)}$ are distributed uniformly on $[0, s]$, the above expectation
can be simplified by first conditioning on $n(s)$, and then summing over all possible values of $n(s)$ (tower property).
\begin{align}
  &  \psi_{\lambda,1}(u, v, s) \nonumber \\
  &= e^{u + vs} \sum_{k = 0}^{\infty} \frac{e^{-\lambda s} (\lambda s)^k} {k!} 
    \expectLone{\prod_{l = 1}^{k} e^{uZ^l(s-R_l) + vA^l(s - R_l)}   }  \nonumber  \\  
 \label{bhatt_1}
    &= e^{u + vs} \sum_{k = 0}^{\infty} \frac{e^{-\lambda s} (\lambda s)^k} {k!} \left( \frac{1}{s} \int_{0}^{s} \psi_\lambda(u, v, \tau) d\tau \right)^k
\end{align}
In the above step, the expectation of the product is the same as the product of the expectations, because the variables
$(Z^l(s-R_l), A^l(s-R_l)), l = 1, \ldots, k$ are independent of each other. Moreover, the expectation of each of the $k$ terms is identical. Denoting
$F(s) \triangleq \int_{0}^{s} \psi_{\lambda,1}(u, v, \tau) d\tau$, we can write \eqref{bhatt_1} as 
\begin{align}
    \dot{F}(s) &= e^{u + vs}e^{-\lambda s} e^{\lambda F(s)} \nonumber \\
  %  e^{-\lambda F(s)}\dot{F}(s) &= e^{(v - \lambda)s + u} \nonumber \\
    \frac{d}{ds}\left(e^{-\lambda F(s)}\right) &= -\lambda e^{(v - \lambda)s + u} ; \quad F(0) = 0 \nonumber \\
    e^{-\lambda F(s)} &= 1 - \lambda e^{u}\int_{0}^{s} e^{(v - \lambda)s'} ds' \nonumber \\
    &= 1 + \frac{\lambda e^{u}}{v - \lambda}\left( 1 - e^{(v - \lambda)s}\right) \nonumber \\
    F(s) &= -\frac{1}{\lambda} \text{log}\left( 1 + \frac{\lambda e^{u}}{v - \lambda}\left( 1 - e^{(v - \lambda)s} \right) \right) \nonumber \\
%    \psi_\lambda(u, v, s)  &= \dot{F}(s) = \frac{e^{(v - \lambda)s + u}} {1 + \frac{\lambda e^{u}}{v - \lambda}\left( 1 - e^{(v - \lambda)s} \right)} 
\end{align}
Finally, using  $\psi_{\lambda,1}(u, v, s)  = \dot{F}(s)$ yields  \eqref{eq:joint_ZA} for $m=1$, and the proof is complete.

\section{Proof of Proposition \ref{prop:small_tau} }  \label{app:small_tau}
\begin{proof}   The basic difficulty to be overcome is that
the limit result $\eta_t \to \eta^*$ in Proposition \ref{prop:gobal_convergence}
doesn't approximately determine the distribution of the  degree evolution for vertex $\tau_o$ if
$\tau_o \not\to \infty.$  To produce an estimator for $\ell_{\tau^o}$ given
 $Y^o_{[\tau^o, T]}$,   we produce a virtual degree growth process,
denoted by  $\breve  Y^{o}_{[\tau, T]},$  which becomes arbitrarily close to
$\tilde Y_{[\tau, T]}$ in total variation distance as $T\to \infty$
under any of the $r$ hypotheses about $\ell_{\tau^o},$  where $\tau \to \infty$
with $\tau /T \to a$ for some fixed  $ \delta > 0.$

Given an arbitrary $\epsilon > 0,$ select $\delta  \in (0,1)$ so small that
$f_Z^{C}(\rho, \theta^*, m,  \ln(1/\delta))  < \epsilon.$
Suppose $\tau$ depends on $T$ such that $\tau / T \to \delta$ as $T\to\infty.$
By Proposition \ref{prop:error_scaling},  $\ell_\tau$ can be recovered
with error probability less than $\epsilon$ from $\tilde Y_{[\tau, T]}$
by using Algorithm C.

The virtual process $\breve Y^o_{[\tau, T]}$ has initial value $\tilde Y^o_\tau =m.$
Thus, although $\tau_o$ arrives before $\tau,$   the virtual process does
not begin evolution until after time $\tau.$   The construction of $\breve Y^o$
proceeds by induction and uses a random thinning of the process
$Y^o$, the actual degree growth process for $\tau^o.$     The thinning probability
is the ratio of degrees.   Specifically,  for $t$ with  $\tau \leq t \leq T-1,$ let
\begin{align*}
\calL( \breve Y^o_{t+1}  - \breve Y^o_{t}|  \breve Y^o_{[\tau,t]},  Y^o_{[\tau_o,T]})  = {\sf binom}\left(Y^o_{t+1}- Y^o_t, \frac{\breve Y^o_t}{Y^o_t} \right).
\end{align*}
The virtual process $\breve Y^o_{[\tau, T]}$  satisfies the same properties as $Y_{[\tau, T]}$
(based on the degree evolution of vertex $\tau$)  used in the proof of Proposition \ref{prop:YJ_couple}, 
so for $v\in [r]$,
$$
d_{TV}\left(  ( \breve Y^o_{[\tau, T]} | \ell_{\tau^o}=v)  ,  (\tilde  Y_{[\tau, T]} | \ell_\tau=v)\right) \to 0.
$$
Hence, applying Algorithm C, designed for recovery of $\ell_\tau$,  to the virtual process
$\breve Y_{[\tau,T]}$ recovers  $\ell_{\tau^o}$ with average error probability less than
$\epsilon$ for $T$ sufficiently large.
\end{proof}

\section{Derivation of the message passing equations}  \label{sec:derivation_of_MP}

The initial conditions given  by  \eqref{eq:initialize_tilde}
are chosen to make the initial likelihood vector the same as produced by
Algorithm C (observation of children).   Equations \eqref{eq:child_to_parent_a} - \eqref{eq:combining}
are derived in what follows in the special case $m=1,$  with the initial graph $G_{t_o}$ consisting
of a single vertex (i.e. $t_o=1$) with a self-loop.   In that case,  the graph $(V,E)$ is a tree
(ignoring the self-loop incident to the first vertex)
so the message passing algorithm is conceptually simpler.   The equations
 \eqref{eq:child_to_parent_a} -  \eqref{eq:combining}
for any finite $m\geq 1$  are simply taken to have the same form as for
$m=1$ on the grounds that loopy message passing is obtained by using the same equations as
for message passing without loops.

Our first assumption in deriving the message passing algorithm
 is that the approximation $\lambda^C_{\tau}$ for the log likelihood
 vector based on observation of children (derived in Section  \ref{sec:recovery_from_children}) is
 exact, or in other words:
\begin{align}\label{eq:LLR}
    \ln \prob{\partial \tau = \{t_1, \ldots, t_n\} \vert \ell_\tau = v}   = \lambda^C_{\tau}(v),
\end{align}
where $\Lambda^c_{\tau}(v)$ is given by \eqref{eq:lambda_eq}.
The second assumption is regarding how the distribution of
$\partial \tau$ changes, given the label of another vertex.   Namely,
  \begin{align}
 & \prob{\partial \tau = \{t_1, \ldots, t_n\} \vert \ell_\tau = v, \ell_{\tau'} = u}  \label{eq:swap}   \\
&  =\left\{ \begin{array}{ll}
   \prob{\partial \tau = \{t_1, \ldots, t_n\} \vert \ell_\tau = v} \theta_{u, v}^* / \theta_v^*  &  \mbox{if } \tau' \in \partial \tau  \\
   \prob{\partial \tau = \{t_1, \ldots, t_n\} \vert \ell_\tau = v}                                                &  \mbox{if } \tau' \not\in \partial \tau 
\end{array} \right. ,  \nonumber
\end{align}
where the expression for the first case follows from \eqref{eq:tildeY_given_labels}.

The third assumption is regarding the joint distribution of degree-growth processes. Observing the degree-growth process of one vertex $\tau$ changes the distribution of the degree growth process
of another vertex $\tau'$ in one of two possible ways. Firstly, the children of the first vertex cannot be the children of the other (if $m = 1$). However,
Proposition \ref{prop:YJ_couple} shows this effect is insignificant. Secondly, observing the degree-growth
process gives us some information about the label of each vertex.
If one vertex appears as a child of the other (say $\tau' \in \partial \tau$), the probability of the
given observation is affected; else it is not. In the asymptotic limit, the degree-growth processes
of a finite number of vertices are indeed independent, by Proposition \ref{prop:YJ_couple}.

The following additional notation is used.   Let $D^k_\tau$ denote the event of observing the subtree of $(V, E)$ rooted at $\tau$, and of depth $k$. For example, $D^1_\tau \equiv \{ \partial \tau = \{t_1, \ldots, t_n\}\} $, $D^2_\tau \equiv \{\partial \tau = \{t_1, \ldots, t_n\}, \partial t_1 = \{t^1_1, \ldots, t^1_{n_1}\}, \ldots, \partial t_n = \{t^n_1, \ldots, t^n_{n_n}\} \}$. Further, let $D_\tau$ denote the event of observing the subtree of $(V, E)$ rooted at $\tau$. We call this subtree as the \emph{descendants} of $\tau$. The event of observing the entire graph is $D_1,$  because the initial graph has a single vertex. Therefore: 
\begin{equation}\label{eq:LLR2}
    \Lambda_{\tau}(v)  = \ln \prob{E_T = E \vert \ell_\tau = v}= \ln \prob{D_1 \vert \ell_\tau = v}
\end{equation}
For a vertex $\tau$ with $\tau \geq 2$, the event $D_1 \backslash D_\tau$
includes the information of which vertex is the parent of vertex $\tau.$
Also, for vertices $\tau$ and $\tau_0$ with $\tau_0 < \tau$,  let $\tau \to \tau_0$ denote
the event there is an edge from $\tau$ to $\tau_0.$

At this point, we make the assumption:
\begin{equation}\label{eq:ind1}
    \prob{D_1 \vert \ell_\tau = v} = \prob{D_{\tau} \vert \ell_\tau = v} \prob{D_1 \backslash D_{\tau} \vert \ell_\tau = v} \ \forall \tau
\end{equation}
In other words,  $D_\tau$ and $D_1\backslash D_\tau$ are assumed to be conditionally independent
given $\ell_\tau = v.$   The rationale for that also comes from ignoring the implications of the
fact that the descendants of $\tau$ must be disjoint from the descendants of
vertices close to $\tau$ in $G_T$ in the direction through the parent of $\tau.$

%The message passing algorithm consists of a single pass of messages sent by each vertex to its parent, followed by a single pass of messages sent by each vertex to each of its children.
Let $\tau$ and $\tau_0$ be vertices such that $\tau$ is a child of $\tau_0$.
We define the messages as follows, and then derive the message passing
equations as fixed points.
%%%%%%
\begin{align} 
& \nu_{\tau \rightarrow \tau_0} (u)  \triangleq \ln \prob{D_{\tau} \vert \ell_\tau = u}  \label{eq:def_nu} \\
& \mu_{\tau_0 \rightarrow \tau}(v) \triangleq \ln \left( \frac{\prob{D_1 \backslash D_\tau \vert \ell_\tau = 0, \ell_{\tau_0} = v}
\theta^*_v}{\theta^*_{0,v}} \right)
     \label{eq:def_mu}  \\
&  \tilde  \nu_{\tau \rightarrow \tau_0} (v) \triangleq \ln \prob{D_{\tau} \vert \ell_{\tau_0} = v, \tau\to\tau_0}  
% +  \ln \theta^*_v   ALT VERSION LEAVES OUT
 \label{eq:def__tilde_nu} \\
&  \tilde \mu_{\tau_0 \rightarrow \tau}(u) \triangleq \ln \prob{D_1 \backslash D_\tau \vert \ell_\tau = u}
     \label{eq:def_tilde_mu}
\end{align}
\begin{remark}
In the definition \eqref{eq:def_mu} of    $\mu_{\tau_0 \rightarrow \tau}$ it is assumed
that 0 represents some choice of label, but the definition 
for all choices of 0 are equivalent.  In  other words, because of \eqref{eq:swap},
\begin{align} 
     \mu_{\tau_0 \rightarrow \tau}(v) = \ln \left(
     \frac{\prob{D_1 \backslash D_\tau \vert \ell_\tau = u, \ell_{\tau_0} = v}\theta^*_v}{\theta^*_{u,v}} \right)
\end{align}
for any $u \in [r].$
\end{remark}

We show that the message passing equations
 \eqref{eq:child_to_parent_a} - \eqref{eq:combining} follow from our independence
 assumptions and  the definitions of the messages given
 in \eqref{eq:def_nu} - \eqref{eq:def_tilde_mu}.
 
\paragraph*{Derivation of \eqref{eq:child_to_parent_a}}   Start with the fact
 $D_\tau = D^1_\tau \cap \left( \cap_{t\in \partial \tau}  D_t \right),$
 and, given $\ell_\tau=v$ and $D^1_\tau$,   The events
 $D_t , t \in \partial \tau$ are conditionally independent.   Hence,
 \begin{align*}
& \prob{D_\tau|  \ell_\tau=v }  \\
&=  \prob{D^1_\tau|  \ell_\tau=v } 
\prod_{t \in \partial \tau}    \prob{D_t|  \ell_\tau=v, D_\tau^1 }    \\
& =  \prob{D^1_\tau|  \ell_\tau=v } 
\prod_{t \in \partial \tau}    \prob{D_t|  \ell_\tau=v,  t\to\tau }.
\end{align*}
So by \eqref{eq:LLR} and the definition of $\tilde \nu_{t\to \tau}$,
\begin{align}
\ln  \prob{D_\tau|  \ell_\tau=v } = \lambda^C_{\tau}(v)  
+ \sum_{t\in \partial \tau}  \tilde \nu_{t \to \tau} (v).   \label{eq:D_ell}
\end{align}
Since $\nu_{\tau\to\tau_0}(v) = \ln  \prob{D_\tau|  \ell_\tau=v }$
this establishes \eqref{eq:child_to_parent_a} for $m=1.$

\paragraph*{Derivation of \eqref{eq:parent_to_child_a}}  Assume
$\tau_0  \geq t_o +1$;   the proof in case $\tau_0 \leq t_o$ is similar.
Then, also accounting for the assumption $m=1,$ \eqref{eq:parent_to_child_a}
becomes
\begin{align}
&  \mu_{\tau_0 \rightarrow \tau}
=  \lambda^C_{\tau_0} + \sum_{t\in \partial \tau_0\backslash \{\tau\} }   \tilde \nu_{t\to\tau_0}
+ \tilde  \mu_{\tau_1 \to \tau_0},
   \label{eq:parent_to_child_a_special}
 \end{align}
 where $\tau_1$ is the parent of $\tau_0.$
 Observe that
\begin{align*}
 D_1 \backslash D_\tau &=  (D_1 \backslash D_{\tau_0}) \cap ( D_{\tau_0} \backslash D_\tau )  \\
&= (D_1 \backslash D_{\tau_0}) \cap D^1_{\tau_0} \cap \left( \cap_{t\in\partial \tau_0 \backslash \{\tau\} } D_t \right)
\end{align*}
 Therefore,
 \begin{align*}
&  \prob{D_1 \backslash D_\tau \vert \ell_\tau = 0, \ell_{\tau_0} = v} \\
& =  \prob{D_1 \backslash D_{\tau_0} \vert \ell_\tau = 0, \ell_{\tau_0} = v} \prob{D_{\tau_0} \backslash D_\tau \vert \ell_\tau = 0, \ell_{\tau_0} = v}  \\
& =  \prob{D_1 \backslash D_{\tau_0} \vert \ell_{\tau_0} = v} \prob{D_{\tau_0} \backslash D_\tau \vert \ell_\tau = 0, \ell_{\tau_0} = v}  \\
& =  \prob{D_1 \backslash D_{\tau_0} \vert \ell_{\tau_0} = v} \prob{D^1_{\tau_0} \vert \ell_\tau = 0, \ell_{\tau_0} = v} \\
&~~~~~~~\prod_{t\in\partial \tau_0 \backslash \{\tau\} }  \prob{D_t \vert  \ell_\tau = 0, \ell_{\tau_0} = v, t \to \tau}  \\
&  =  \prob{D_1 \backslash D_{\tau_0} \vert \ell_{\tau_0} = v} \prob{D^1_{\tau_0} \vert \ell_\tau = 0, \ell_{\tau_0} = v} \\
&~~~~~~~\prod_{t\in\partial \tau_0 \backslash \{\tau\} }  \prob{D_t \vert \ell_{\tau_0} = v, t\to \tau}
\end{align*}
Multiplying both sides of the above by $\frac{\theta_v^*}{\theta^*_{0,v}}$, using \eqref{eq:swap},
 and taking logarithms yields
\begin{align*}
 &\mu_{\tau_0 \rightarrow \tau}(v)   =  \ln \frac{\theta_v^*}{\theta^*_{0,v}} +
  \tilde  \mu_{\tau_1 \to \tau_0}(v) + \left( \Lambda_{\tau_0}^C(v) + \ln \frac{\theta^*_{0,v}}{\theta^*_v}  \right)   \\
  &~~~+ \sum _{t\in\partial \tau_0 \backslash \{\tau\} }  \tilde \nu_{t\to \tau}(v) ,
 \end{align*}
which is equivalent to  \eqref{eq:parent_to_child_a_special}, so that \eqref{eq:parent_to_child_a} is proved for $m=1.$

\paragraph*{Derivation of \eqref{eq:child_to_parent_b}} Note that
\begin{align*}
&\prob{D_\tau | \ell_{\tau_0}=v, \tau\to\tau_0} \\
&= \sum_{u\in [r]} \prob{D_\tau,  \ell_{\tau}=u  | \ell_{\tau_0}=v, \tau\to\tau_0}  \\
&=  \sum_{u\in [r]}
\prob{\ell_{\tau}=u  | \ell_{\tau_0}=v, \tau\to\tau_0} 
\prob{D_\tau | \ell_{\tau}=u}  \\
& = \sum_{u\in [r]}  \frac{\rho_u \theta^*_{u,v}}{\theta^*_v}  \eexp^{\nu_{\tau\to\tau_0}(u)}, 
\end{align*}
where for the second inequality we used \\
 $ \prob{D_\tau,  \ell_{\tau}=u  | \ell_{\tau_0}=v, \tau\to\tau_0} 
= \prob{D_\tau | \ell_{\tau}=u} .$
Taking the logarithm of each side yields  \eqref{eq:child_to_parent_b}.

\paragraph*{Derivation of \eqref{eq:parent_to_child_b}}   The derivation is given by:
\begin{align*}
&\tilde \mu_{\tau_0 \to \tau} (u)  = \ln \prob{D_1\backslash D_\tau | \ell_\tau  =u }  \\
&=\ln  \sum_{v \in[r]}  \prob{D_1\backslash D_\tau, \ell_{\tau_0}=v  | \ell_\tau =u } \\
&=\ln  \sum_{v\in[r]} \prob{D_1\backslash D_\tau  |  \ell_\tau =u,  \ell_{\tau_0}=v }
  \prob{ \ell_{\tau_0}=v  | \ell_\tau =u }  \\
 &=\ln  \sum_{v \in[r]}\theta_{u,v}^*  
 \frac{ \prob{D_1\backslash D_\tau  |  \ell_\tau =u, \ell_{\tau_0}=v } }{\theta_{u,v}^*}  \rho_{v} \\
  &=\ln  \sum_{v \in[r]}\theta_{u,v}^*  
 \frac{ \prob{D_1\backslash D_\tau  |  \ell_\tau =0, \ell_{\tau_0}=v } }{\theta_{0,v}^*}  \rho_{v} \\
 & = g^{pc}(\mu_{\tau_0\to\tau})(u).
\end{align*}

\paragraph*{Derivation of \eqref{eq:combining}}
Equation \eqref{eq:combining} (for $m=1$) follows from \eqref{eq:LLR2},
\eqref{eq:ind1},  \eqref{eq:D_ell},  and \eqref{eq:def_tilde_mu}.

\bibliographystyle{IEEEtran}
\bibliography{../graph_inference}

\end{document}